\documentclass[11pt]{article}
\usepackage{fullpage}
\usepackage[round]{natbib}

\newif\ifsub
\subfalse

\renewcommand{\cite}{\citep}

\usepackage[utf8]{inputenc} %
\usepackage[T1]{fontenc}    %
\usepackage{hyperref}       %
\usepackage{url}            %
\usepackage{booktabs}       %
\usepackage{amsfonts}       %
\usepackage{nicefrac}       %
\usepackage{microtype}      %
\usepackage{xcolor}         %
\usepackage{subcaption}
\usepackage{amsmath}   
\usepackage{amsthm}   
\usepackage{graphicx}

\usepackage{amsmath}
\usepackage{amssymb,wasysym}
\usepackage{mathtools}
\usepackage{amsthm}

\usepackage{subcaption}
\usepackage{wrapfig}

\usepackage[capitalize,noabbrev]{cleveref}

\usepackage{thmtools} 
\usepackage{thm-restate}

\hypersetup{
  colorlinks   = true, %
  urlcolor     = blue, %
  linkcolor    = blue, %
  citecolor   = blue %
}
\usepackage{authblk}

\usepackage{thmtools} 
\usepackage{thm-restate}
\usepackage[parfill]{parskip}

\usepackage{tikz}

\title{Emergent specialization from participation dynamics and multi-learner retraining}

\date{June, 2022, \emph{updated April, 2024}}

\newcommand\blfootnote[1]{%
  \begingroup
  \renewcommand\thefootnote{}\footnote{#1}%
  \addtocounter{footnote}{-1}%
  \endgroup
}

\author[1]{Sarah Dean}
\author[2]{Mihaela Curmei}
\author[3]{Lillian J. Ratliff}
\author[3]{Jamie Morgenstern}
\author[3]{Maryam Fazel}
\affil[1]{Cornell University}
\affil[2]{University of California, Berkeley}
\affil[3]{University of Washington}

\begin{document}

\maketitle
\newcommand{\numgroup}{n}
\newcommand{\numplayer}{m}
\newcommand{\indgroup}{i}
\newcommand{\indplayer}{j}
\newcommand{\group}{subpopulation}
\newcommand{\groupm}{\mathsf{subpop}}
\newcommand{\groups}{subpopulations}
\newcommand{\player}{learner}
\newcommand{\playerm}{\mathsf{learner}}
\newcommand{\players}{learners}
\newcommand{\groupsize}{\beta}
\newcommand{\participation}{participation}
\newcommand{\partic}{\alpha}
\newcommand{\groupalpha}{\alpha_{i, :}}
\newcommand{\playeralpha}{\alpha_{:, j}}
\newcommand{\allocation}{allocation}
\newcommand{\allocations}{allocations}
\newcommand{\alloc}{\nu}
\newcommand{\eq}{\mathsf{eq}}
\newcommand{\total}{\mathsf{total}}
\newcommand{\partition}{\mathcal{S}}
\newcommand{\alltheta}{\Theta}

\newcommand{\distr}{\mathcal{D}}
\newcommand{\risk}{\mathcal{R}}
\newcommand{\loss}{\ell}
\newcommand{\diag}{\mathrm{diag}}
\newcommand{\Diag}{\mathrm{Diag}}
\newcommand{\tr}{\mathrm{tr}}

\newcommand{\mc}{\mathcal}
\newcommand{\mb}{\mathbb}

\newcommand{\E}{\mathbb{E}}
\newcommand{\R}{\mathbb{R}}
\renewcommand{\P}{\mathbb{P}}

\newtheorem{theorem}{Theorem}[section]
\newtheorem{lemma}[theorem]{Lemma}
\newtheorem{prop}[theorem]{Proposition}
\newtheorem{coro}[theorem]{Corollary}
\newtheorem*{remark}{Remark}
\newtheorem{corollary}{Corollary}[theorem]

\theoremstyle{definition}
\newtheorem{definition}[theorem]{Definition}
\newtheorem{example}[theorem]{Example}

\newcommand{\sd}[1]{\textcolor{teal}{[SD: #1]}}
\newcommand{\mi}[1]{\textcolor{blue}{[MC: #1]}}
\newcommand{\mic}[1]{\textcolor{blue}{[#1]}}
\newcommand{\mf}[1]{\textcolor{magenta}{[MF: #1]}}

\newcommand{\propchange}[1]{\textcolor{blue}{#1}}
\renewcommand{\propchange}[1]{{#1}} \blfootnote{$^1$sdean@cornell.edu}

\begin{abstract}
Numerous online services are data-driven: the behavior of users affects the system's parameters, and 
the system's parameters affect the users' experience of the service, which 
in turn affects the way users may interact with the system. 
For example, people may choose to use a service only for tasks that already works well, or they may choose to switch to a different service. 
These adaptations influence the ability of a system to learn about a population of users and tasks in order to improve its performance broadly. 
In this work, we analyze a class of such dynamics---where users allocate their participation amongst services to reduce the individual risk they experience, and services update their model parameters to reduce the service's risk on their current user population. We refer to these dynamics as 
\emph{risk-reducing}, which cover a broad class of common model updates including gradient descent and multiplicative weights. 
For this general class of dynamics, we show that asymptotically stable equilibria are always segmented, with sub-populations allocated to a single learner. Under mild assumptions, the utilitarian social optimum is a stable equilibrium. 
In contrast to previous work, which shows that repeated risk minimization can result in  representation disparity and high overall loss %
with a single learner  \citep{hashimoto2018fairness,miller2021outside}, 
we find that repeated myopic updates with multiple learners lead to better outcomes. 
We illustrate the phenomena via a simulated example initialized from real data.
\end{abstract}

\ifsub
    \section{INTRODUCTION}
\else
    \section{Introduction}
\fi

Many online platforms, including social media networks, personalized recommendation engines, and advertising auction systems, 
  collect user data and make incremental adjustments to the models they use to personalize content. These continuous updates are motivated by many factors, though large amongst them is the fact that the systems operate in non-stationary environments, where the preferences of their users change as the system operates. 
Changes in user preferences might occur exogeneously of service settings (e.g., global events might spur interest in new topics) or endogeneously (e.g., increasing the ranking of certain content on a platform might lead the content to ``go viral"). 
The fact that user behavior might depend on service settings can take on many forms: people may learn to ignore or avoid clicking on advertisements; they may choose to use the service only for tasks at which it already works well; or they may choose to switch to a different service if they have a better experience with the second service. The latter two examples of adaptation, where users might opt for services that already suit their needs, affect the system's capacity to learn about its user base and improve its overall performance.

In this work, we study a particular form of endogenously shifting distributions
over multiple rounds, in contexts where individuals prefer to use services whose predictions are more accurate for them. 
Much of the existing work on endogeneous distribution shift focuses on users who modify their features to achieve desired outcomes, as in strategic classification 
\citep{hardt2016strategic} and related problems 
\citep{perdomo2020performative,miller2021outside}. While important, this model of data
manipulation does not capture the most straightforward way that individuals express their preferences in a market: by choosing amongst alternative providers. 
In fact, recent work has shown that in the presence of a choice of participation between competing providers, individuals do not have an incentive to perform costly data
manipulations~\citep{hardt2022performative}.

 Consider as an example a social media platform.  If
the platform recommends content that does not appeal to the tastes of younger generations, these users will spend a smaller fraction of their time on that platform. 
This results in positive (i.e., self-reinforcing) feedback
loop, where a services's poor performance on young customers dissuades
them from using the service, leading to less data and
diminishing weight placed on making better predictions for young
customers in the future.  Within a single service, these
effects may lead to representation
disparity~\citep{hashimoto2018fairness}.

However, in a broader ecosystem, individuals can  choose \emph{amongst} services.  
If a new social media platform can predict the tastes of younger users more accurately, the younger users may spend more of their time on the new service, and correspondingly less on an existing platform.
The new platform will then receive more data and improve its performance on young customers, while the old platform's predictions may deteriorate, reinforcing their exit. Similarly, in the context of Large Language Models (LLM), if one 
LLM performs particularly well on creative tasks and another on answering homework questions, the distribution of prompts each receives may shift towards their existing expertise.
Such feedback loops can also arise in settings such as music recommendation or healthcare, where demographic and socio-economic factors explain some of the emerging specialization (see examples in Appendix \ref{app:motivation}).
In this paper, we study the dynamics of
populations apportioning themselves amongst services, and services that
choose predictors based on their observed user population.
Our {first contribution} is to introduce and formalize this general setting. In
Section~\ref{sec:setting} we introduce \emph{risk reducing} populations and services who choose their actions myopically,
incrementally improving their utility based on current conditions.
Our {second contribution} is to present a complete characterization of stable
fixed points for this general class of dynamics in Section~\ref{sec:main_results}. 
By drawing a connection between the dynamics and the \emph{total risk}, our {third contribution} is to
characterize the implications of this dynamic in terms of a
utilitarian notion of \emph{social welfare}, and argue that increasing
the number of available services leads to better outcomes in
terms of accurate predictions and user experience.  In
Section~\ref{sec:experiments} we illustrate our theory with simulated experiments and conclude with a discussion of future work in Section~\ref{sec:conclusion}.

\ifsub
   \section{RELATED WORK}
\label{sec:relatedwork}
\else
   \section{Related Work}
\label{sec:relatedwork}
\fi

\label{sec:relatedwork}

The study of equilibria in the presence of utility optimizing agents has classical roots in game theory, and
optimization over
decision-dependent probabilities
is classically studied by stochastic
optimization and control (e.g., the review article by
\citet{hellemo2018decision});
we narrow
our focus to the most relevant literature on this %
as it arises in machine learning systems.

\textbf{Endogenous Distribution Shifts.}  In the study of machine
learning systems, a large body of literature studies
exogenous distribution shifts such as covariate, label, or concept
drift~\citep{quinonero2008dataset}. A more recent trend is to study
shifts in the underlying data distribution due to %
endogenous reactions, for example due to strategic
behavior exhibited by a user population. 
The work of
\citet{perdomo2020performative} introduces 
\emph{performative prediction} as a model capturing user reaction via endogenous distribution shifts. This work models a single decision-maker 
facing a risk minimization problem subject to an 
underlying decision-dependent data distribution. Following its
introduction, several relevant solution concepts have been explored and algorithms for achieving them proposed
\citep{izzo2021learn,drusvyatskiy2020stochastic,
  mendler2020stochastic,miller2021outside}. 
  A variant of the single
decision-maker performative prediction problem studies
  time-dependent
dynamics of the data distribution,  
with both exogenous~\citep{wood2021online,cutler2021stochastic} and endogenous
\citep{ray2022decision,brown2022performative} sources.
These works primarily consider strategic covariate shifts in a single distribution.
In contrast, we consider a mixture of distributions: sub-populations of %
users whose participation choices result in 
attrition and retention dynamics which are not studied in the aforementioned distribution shift literature.
\textbf{Multiple Decision-Makers.}  Endogenous distribution shift has
also been studied in settings with multiple decision-makers as a
continuous game. For instance, the multi-player performative
prediction problem extends the original problem by allowing for
multiple competing decision-makers \citep{narang2022multiplayer,
  piliouras2022multi, wood2022stochastic}.  
This line of work
differs from ours in that the population is modeled as homogeneous and stateless.
These works focus on characterizing the existence and uniqueness of
different types of competitive equilibria for the game, and
analyze %
learning dynamics that lead to different
equilibrium concepts.
In contrast, in our paper the focus is on asymptotically stable points
(equilibrium) for the combined dynamical system resulting from myopic optimization by non-anticipating decision-makers and stateful user
participation updates.

\textbf{Retention.} 
User retention in machine learning systems
is closely related to the
population participation dynamics we consider
\citep{hashimoto2018fairness,zhang2019group}.  In settings with
multiple sub-populations of users of different types, %
the question of retention has been explored in parallel with the issue
of fairness. \citet{hashimoto2018fairness} coined the term
\emph{representation disparity} for the phenomenon in which the
traditional approach of minimizing average performance leads to high
overall accuracy coupled with low accuracy on minority groups, causing an exodus of said groups. 
For single learners, systems which instead perform robust risk minimization avoid
such disparity.

Our work generalizes the single-learner retention setting and analyzes the fixed points of dynamics between
multiple systems and populations without modifying risk functions to be robust. 
\citet{ginart2021competing} also consider user choice between multiple learning systems, with an empirical investigation
and  theoretical results in restricted settings focused on finite sample effects.
In contrast, we propose a general class of  \emph{risk reducing} dynamics and
develop a comprehensive theoretical understanding.

\ifsub
    \section{FRAMEWORK AND SETTING}\label{sec:setting}
\else
    \section{Framework and Setting}\label{sec:setting}
\fi

We consider a setting where  the population of individuals is composed of $\numgroup$ \groups{}
spreading their participation amongst  $\numplayer$ \players{} (service providers or decision-makers).
Figure~\ref{fig:cartoon} illustrates a simple example.
Each \group{} $\indgroup\in\{1,\dots,\numgroup\}=: [\numgroup]$ has features and labels distributed according to a fixed distribution $(x,y)=: z \sim \distr_\indgroup$ and makes up $\groupsize_\indgroup$ proportion of the total population, so that $\sum_{\indgroup=1}^\numgroup \groupsize_\indgroup = 1$.
An $\partic_{\indgroup\indplayer}$ proportion of
\group{} $\indgroup$ is associated to each \player{} $\indplayer\in[\numplayer]$, normalized so that
$\sum_{\indplayer=1}^\numplayer \partic_{\indgroup\indplayer} = 1$. The \groups{} therefore 
redistribute their \participation{} among the various \players{}. Further, to model the ability of \groups{} to opt-out, one can include  a static ``null \player{}''.%

\begin{figure}[h]
  \begin{center}
    \includegraphics[width=.5\textwidth]{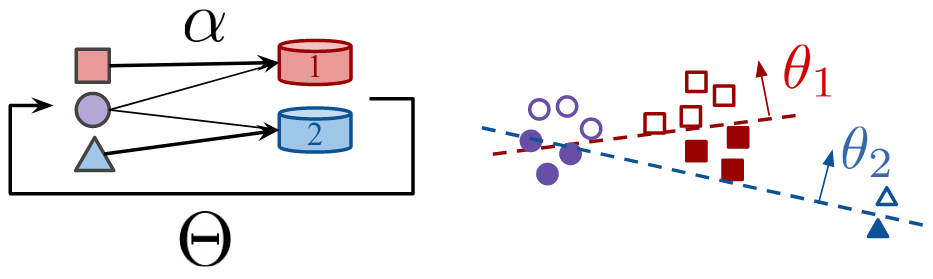}
  \end{center}
    \caption{$n=3$ \groups{} ($\square$, $\ocircle$, $\triangle$) select among $m=2$ \players{} (red, blue) based on classification accuracy with respect to label (solid, hollow). Parameters $\Theta=(\theta_1,\theta_2)$ (decision lines) update in response to \groups{} participation $\alpha_{i,j}$. At the current state, the circle \group{} will shift participation towards blue \player{}.}
    \label{fig:cartoon}
\end{figure}

A \group{} can be as broad as a demographic or affinity group and as granular as a single individual.
The allocation of a \group{} can represent several things: the fraction of a \group{} which uses a given service, or the \emph{fraction of time} users from that \group{} choose to spend using  \players{}' systems.
Accordingly, the relative size $\beta_i$ of the population can represent the proportion of individuals or the \emph{total time} individuals spend. 
This framework also allows for a \group{} to represent \emph{types of tasks or activities} a user wishes to accomplish, allocating these tasks 
to \players{} based upon which systems perform best on which tasks.
\emph{The only assumption we make about any \group{} is that individual samples comprising it are i.i.d.} 

Throughout, we assume that there are fewer \players{} than \groups{}, $\numplayer\leq \numgroup$.
Each \player{} $\indplayer$ observes data from the \groups{} who participate in it.
Formally it observes features and labels drawn from the mixture distribution determined by the \participation{} and \group{} sizes:
$$(x,y)_j = z_j\sim \frac{\sum_{\indgroup=1}^\numgroup  \partic_{\indgroup\indplayer} \groupsize_\indgroup \distr_\indgroup}{\sum_{\indgroup=1}^\numgroup\partic_{\indgroup\indplayer} \groupsize_\indgroup}$$
Learners make predictions or decisions according to a parameter $\theta_\indplayer\in\R^d$.
Beyond the information encoded in the features and labels, the \players{} are unaware of which \group{} individual data points are.

The quality of predictions made by parameter $\theta_j\in\R^d$ for an individual instance $z_j$  is quantified by the loss $\loss(\theta_j; z_j)$. 
The quality of $\theta_j$ for a \group{} is quantified by the average loss, i.e. the \emph{risk} $\risk_\indgroup(\theta_j)=\E_{z\sim \distr_\indgroup}[\loss(\theta_j; z)]$.  %
Throughout, we will make the additional assumption that the risk function for each \group{} $\risk_\indgroup(\theta)$ is convex and differentiable. %
Figure~\ref{fig:lin_2d_ex} illustrates an example of the risk functions arising in linear regression.

\begin{figure*}[h]
    \centering
    \includegraphics[width=0.32\textwidth]{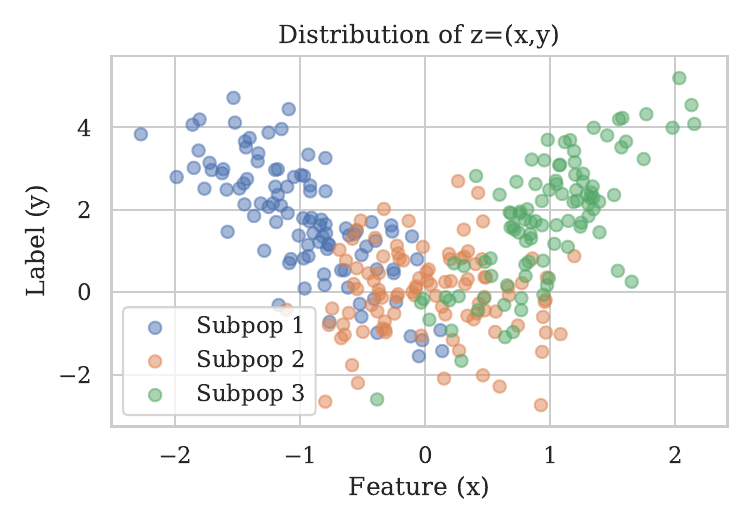}
    \includegraphics[width=0.32\textwidth]{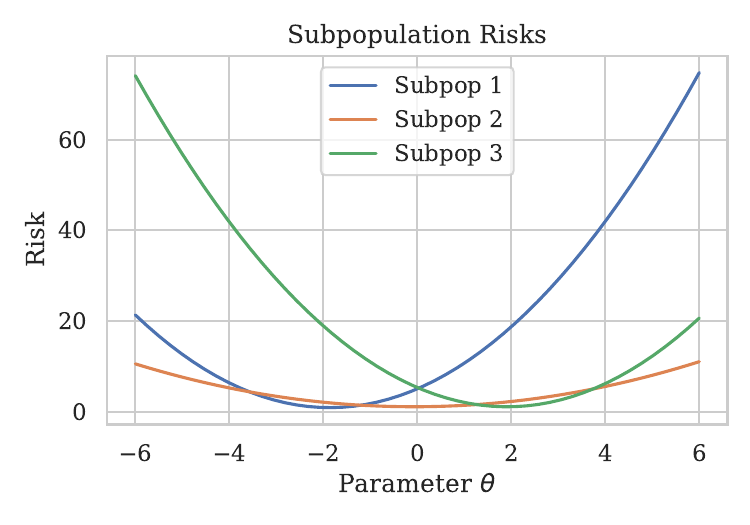}
    \includegraphics[width=0.32\textwidth]{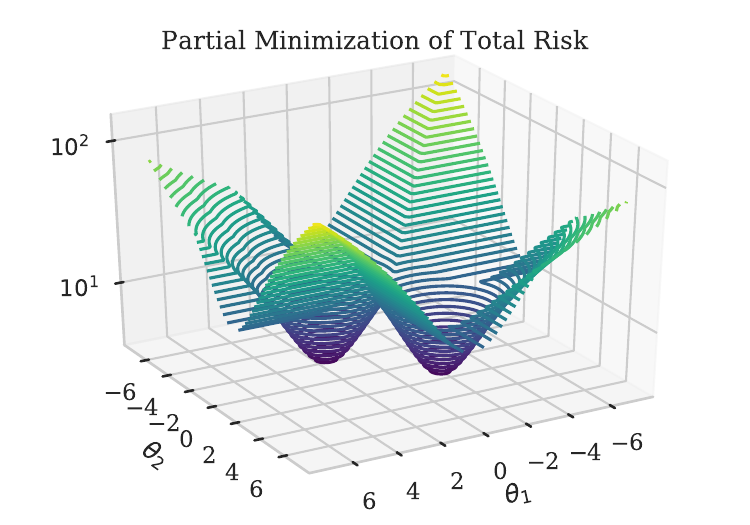}
    \caption{An example arising from least-squares linear regression with $n=3$ \groups{} and $m=2$ \players{}. Left: The distribution of $z=(x,y)$, colored by \group{}. Middle: The \group{} risks $\risk_\indgroup(\theta)$ arising from least-squares linear regression $\ell(\theta; z) = (y-\theta x)^2$. Right: A visualization of the non-convex total risk as a function of \player{} parameters, via the partial minimization over \group{} allocation: $\min_{\alpha} \sum_{i=1}^3 \sum_{j=1}^2 \alpha_{ij} \mathcal R_i(\theta_j) = \sum_{i=1}^3 \min\{\mathcal R_i(\theta_1), \mathcal R_i(\theta_2)\}$. }
    \label{fig:lin_2d_ex}
\end{figure*}

\subsection{Decision dynamics of \players{} and \groups{}}\label{sec:dynamics}
Subpopulations and \players{} react to each other; Updates in \group{} allocations lead to updates in \players{} parameters $\Theta^t = (\theta^t_1,\dots,\theta^t_m)$, and vice versa.
We introduce a broad class of update dynamics by way of a canonical example.
Suppose that each \group{} $i$ updates its allocation by increasing the participation proportional to the quality of various models; for example, by spending more time on recommendation platforms that suggest more engaging content.
Recalling that the risk (i.e. average loss) quantifies quality, this manifests as a \emph{multiplicative weights update}:  $\partic_{\indgroup\indplayer}^{t+1}  \propto \partic_{\indgroup\indplayer}^t\cdot \exp(-\gamma  \risk_\indgroup(\theta_\indplayer))$ for $j\in[m]$ and some parameter $\gamma>0$.
This is similar to the retention function studied by~\citet{hashimoto2018fairness} and has connections to {replicator dynamics}, a  foundational evolutionary dynamic that can be interpreted as a process of information diffusion and imitation~\citep{sandholm2020evolutionary}.

Recall that each \player{} $j$ observes data from the mixture distribution $(\sum_{\indgroup=1}^\numgroup\partic_{\indgroup\indplayer} \groupsize_\indgroup)^{-1}\sum_{\indgroup=1}^\numgroup  \partic_{\indgroup\indplayer} \groupsize_\indgroup \distr_\indgroup$ for which we use the shorthand $\mathcal D(\alpha_{:,j})$, where $\playeralpha \in\R^\numgroup$ denotes the vector of allocations from all \groups{} to \player{} $\indplayer$.
Suppose the \players{} update their parameters using gradient descent to reduce the average loss over this data (e.g. to improve the prediction of user engagement). 
Setting aside finite sample issues, for a step size $\gamma_t$ the gradient update takes the form $\theta_j^{t+1} =  \theta_j^t - \gamma_t \nabla_{\theta}  \mathop{\mathbb{E}}_{z\sim \mathcal D(\alpha_{:,j})}\left[\ell(\theta_j^t; z)\right]$.
This is an incremental version of the 
\emph{repeated retraining dynamics} which have been studied in the single learner setting by~\citet{hashimoto2018fairness,perdomo2020performative}.

Despite the apparent simplicity of independent update rules, the evolution of \groups{} and \players{} is highly coupled. The sequential interaction between \groups{} and \players{} leads to complex nonlinear dynamics: i.e. multiplicative weights over non-stationary risks (due to \player{} updates) and gradient descent over non-stationary data distributions (due to \group{} updates).
To study this complex behavior, we now formalize key properties.

The first observation is that updates are \emph{ stateful}, with \group{} allocations and \player{} parameters depending on previous values.
This motivates a description of the
dynamics arising from interactions between $n$ \groups{} and $m$ \players{} in terms of the overall state $\partic\in \Delta_m\times \dots \times \Delta_m =: \Delta_m^n$ and $\alltheta \in \R^{m\times d}$.
We thus define for each \group{} $i$ a general Markovian \allocation{} function $\alloc^t_\indgroup : \Delta_{\numplayer}\times \R^{ \numplayer\times d} \to  \Delta_{\numplayer}$ which describes the participation update $\groupalpha^{t+1} = \alloc^t_\indgroup \left(\groupalpha^{t}, \alltheta^t \right)$ at time $t$, where $\groupalpha\in \Delta_\numplayer$ denotes the vector of allocations from the \group{} $\indgroup$ to all \players{}.
Similarly, define $\mu^t_j: \R^{d} \times  \R^{n} \to  \R^{d}$ so
\player{} $j$ updates their parameter according to $\theta^{t+1}_\indplayer = \mu^t_{\indplayer}(\theta^t_\indplayer, \playeralpha^t)$.

The second observation is that the basis for the updates is the average loss, i.e. risk. This motivates the following definition: given participation $\alpha$ and parameters $\Theta$,
the average risk experienced by each \group{} $i$ and each 
\player{} $j$ is:
\begin{align*}
   \bar\risk^{\groupm}_\indgroup \left(\groupalpha, \alltheta \right) 
   &:= \mathop{\mathbb{E}}_{j\sim \alpha_{i,:}}\left[\mathop{\mathbb{E}}_{z\sim \mathcal D_i}\left[\ell(\theta_j;z)\right]\right],
    \\
    \bar\risk^{\playerm}_\indplayer\left(\playeralpha, \theta_{\indplayer} \right) &:= \mathop{\mathbb{E}}_{z\sim \mathcal D(\alpha_{:,j})}\left[\ell(\theta_j; z)\right].
\end{align*}
In the recommendation example,
$\bar\risk^{\groupm}$ captures the dissatisfaction with content for a \group{} and 
$\bar\risk^{\playerm}$ corresponds to the average prediction error of the platform.
Intuitively, multiplicative weights reduces the average \group{} risk while gradient descent reduces the average \player{} risk.

\begin{definition}[Reducing and Minimizing Dynamics]
A $u$ update rule is \emph{$P$-reducing} w.r.t. $v$
if $P(u^{t+1},v^t) \leq P(u^t, v^t)$ for all $t$ and any sequence of $v^t$. %
It is further \emph{{$P$-minimizing in the limit}} if
the inequality is strict when $u^{t}$ is not a minimizer and $\lim_{t\to\infty} P(u^{t}, v) = \min_{u} P(u, v)$.
\end{definition}

We call a \group{} $i$ \emph{risk reducing} (resp. minimizing) when the allocation update on $\partic_{i,:}$ is $\bar\risk^{\groupm}_\indgroup$-reducing (resp. minimizing in the limit) with respect to $\Theta$.
Similarly, we call a \player{} $j$ \emph{risk reducing} (resp. minimizing) when the parameter update on $\theta_j$ is $\bar\risk^{\playerm}_\indplayer$-reducing (resp. minimizing in the limit) with respect to $\partic_{:,j}$.

We remark that the notion of risk minimizing in the limit is reasonable for \groups{} because their average risk is linear in $\alpha_{i,:}$. It is also reasonable for \players{} because their average risk is convex in $\theta_j$ (due to the assumption that risks $\mathcal R_i(\theta_j)$ are convex). 
However, risk-reducing/minimizing is only a property defined with respect to the participation $\alpha$ or parameter $\Theta$ observed at a previous time step. Thus it does not necessarily hold that $\bar\risk^{\playerm}_{\indplayer}$ or  $\bar\risk^{\groupm}_{\indgroup}$ decrease when the state evolves $(\alpha^{t}, \alltheta^{t}) \to (\alpha^{t+1}, \alltheta^{t+1})$ by sequential updates of $\nu^t$ and $\mu^t$. Our experiments (Figure \ref{fig:cost_curves}) illustrate the non-monotonicity of the coupled updates.

\begin{example}[Semi-static participation]
Suppose a population has a constant allocation of 20\% to one learner, while the remaining 80\% is allocated to the remaining learners inversely proportional to the learner's risk on that population. This is risk reducing but not risk minimizing in the limit.
\end{example}

\begin{example}[Full risk minimization]\label{ex:full_player}
Suppose that a \player{} updates its parameter to minimize the average risk function
$\bar\risk^{\playerm}_{\indplayer}(\playeralpha^t, \cdot)$ at each timestep.
This has been studied as \emph{repeated retraining dynamics} in the single learner case by~\citet{hashimoto2018fairness,perdomo2020performative}.
\end{example}

\begin{restatable}{prop}{ExampleRR}
\label{prop:example_rr}
A \group{}  $i$ updating their participation with multiplicative weights is risk minimizing in the limit if $\gamma>0$ and $\partic^{0}_{\indgroup\indplayer}>0$  $\forall j$.
A \player{} updating its parameter with gradient descent is risk minimizing in the limit when 
the risk functions $\mathcal R_i(\theta)$ are $L$ smooth and the step size satisfies $\gamma^t<\frac{2}{L}$,
$\sum_{t=0}^\infty \gamma^t  = \infty$, and  $\sum_{t =1}^\infty (\gamma^t)^2 < \infty$. 
\end{restatable}
We provide a proof and detail several other examples of risk reducing dynamics in Appendix~\ref{sec:appendix}.

\subsection{Equilibria and stability}

We focus on the equilibrium states resulting from risk-reducing  \groups{} and \players{}.

\begin{definition}[Equilibrium] The state $(\partic^\eq, \alltheta^\eq)$ is an equilibrium state if it is stationary under the dynamics update $\{\nu_i^t\}$, $\{\mu_j^t\}$; i.e. that for all $\indgroup \in [\numgroup]$ and  $ \indplayer \in [\numplayer]$:
\[\groupalpha^\eq = \alloc^t_\indgroup(\groupalpha^\eq, \alltheta^\eq)\quad\text{and}\quad \theta_{\indplayer}^\eq = \mu^t_{\indplayer}(\playeralpha^\eq, \theta_{\indplayer}^\eq)\:.\]
\end{definition}
If \players{} and \groups{} are in an equilibrium state, they will remain that way indefinitely. 
However, some equilibrium configurations may be unstable to perturbations. 
\begin{definition}[Stable Equilibrium] The state $(\partic^\eq, \alltheta^\eq)$ is a \emph{stable} equilibrium state if it is an equilibrium and for each $\epsilon_\partic, \epsilon_\theta>0$, there exist $\delta_{\partic} , \delta_{\theta}> 0$ such that  \[\begin{matrix} \Vert \partic^0 - \partic^\eq \Vert < \delta_{\partic},\\ \Vert \alltheta^0 - \alltheta^\eq \Vert < \delta_{\theta} \end{matrix} \implies \begin{matrix} \Vert \partic^t - \partic^\eq \Vert \leq\epsilon_\partic,\\  \Vert \alltheta^t - \alltheta^\eq \Vert \leq \epsilon_\theta\end{matrix}~\forall~t\geq 0.\] 
It is further \emph{asymptotically stable} if $\lim_{t\to \infty} \Vert \partic^t - \partic^\eq \Vert = 0$ and $\lim_{t\to \infty} \Vert \alltheta^t - \alltheta^\eq \Vert = 0$.
\end{definition}
Stability analysis identifies qualitatively different equilibrium states.
For the class of risk reducing dynamics that we study, equilibria may be unstable, stable, or asymptotically stable;
Appendix \ref{sec:stability ex} presents examples.
While a quantitative understanding of convergence may also be of interest, it would require stronger assumptions on the behavior of \groups{} and \players{}; here we favor generality and leave this to future work.
Furthermore, characterizing stable equilibria sets the foundation for understanding high probability behavior of systems under noisy updates which are risk reducing only in expectation~\citep{kushner1967stochastic}.
This sets the stage for finite sample risk minimization or multi-agent user models, a challenge which we leave to future work.

\ifsub
    \section{MAIN RESULTS}\label{sec:main_results}
\else
    \section{Main Results}\label{sec:main_results}
\fi
We study a large class of feedback dynamics between risk reducing \players{} and \groups{} described by the sequential updates: $\alpha^{t+1} = \alloc^t(\alpha^t,\alltheta^t)$ and $\alltheta^{t+1} = \mu^t(\alpha^{t+1},\alltheta^t)$.
Our analysis allows for \players{} and \groups{} who exhibit a diverse range of behaviors.
We do not require that every \player{} or every \group{} update their parameter or allocation in the same manner or even at every timestep, allowing for any number of round-robin schemes.
Our only assumption on \player{} and \group{} updates is that they are  risk reducing or minimizing.

Figure~\ref{fig:summary} presents a summary of the equilibria characterization that we present in this section.
All omitted proofs can be found  in Appendix \ref{app:results}.

\ifsub
\begin{figure}[h]
\centering
\begin{tikzpicture}[
    node/.style={%
      draw,
      rectangle,
      rounded corners,
      align=center,
      font=\small,
    },
    leaf/.style={
      draw=red,
      fill=red!10,
      rounded corners,
      align=center,
      font=\small
    },
    decision/.style={
      midway,
      fill=white,
      font=\footnotesize,
    }
  ]
    \node [node] (A1) {$\Theta_0\in\arg\min_\Theta \risk^\mathsf{total}(\alpha_0,\Theta)$?};
    \path (A1) ++(-100:16mm) node [node] (A2) {$\Theta_0$ unique\\minimizer?};
    \path (A2) ++(-100:17mm) node [node] (A) {$\alpha_0$ segmented?};
    \path (A1) ++(-29:33mm) node [leaf] (A3) {not an\\ equilibrium};
    \path (A2) ++(-29:33mm) node [leaf] (A4) {no general classification\\(Appendix \ref{sec:stability ex})};
    \path (A) ++(-100:17mm) node [node] (B) {Strict preference?\\(Eq. (1) in Thm. \ref{thm:segmented_market_eq})};
    \path (B) ++(-110:19mm) node [leaf] (F) {asymptotically\\ stable};
    \path (B) ++(-55:20mm) node [leaf] (I) {unstable};
    \path (A) ++(-28:35mm) node [node] (C) {Support is over\\ \emph{risk equivalent} learners?\\ (Eq. (2) in Thm. \ref{thm:balanced_eq})};
    \path (C) ++(-115:31mm) node [node] (D) {Support is over\\ \emph{risk optimal} learners?\\ (Eq. (2) in Thm. \ref{thm:balanced_eq})};
    \path (C) ++(-60:32mm) node [leaf] (E) {not an \\equilibrium};
    \path (D) ++(-55:22mm) node [leaf] (G) {unstable};
    \path (D) ++(-130:25mm) node [leaf] (H) {may be stable\\ (balanced equilibrium)};

    \draw (A1) -- (A2) node [decision] {yes};
    \draw (A1) -- (A3) node [decision] {no};
    \draw (A2) -- (A) node [decision] {yes};
    \draw (A2) -- (A4) node [decision] {no};
    \draw (A) -- (B) node [decision] {yes};
    \draw (A) -- (C) node [decision] {no};
    \draw (B) -- (F) node [decision] {yes};
    \draw (B) -- (I) node [decision] {no};
    \draw (C) -- (D) node [decision] {yes};
    \draw (C) -- (E) node [decision] {no};
    \draw (D) -- (H) node [decision] {yes};
    \draw (D) -- (G) node [decision] {no};
\end{tikzpicture}
\caption{A summary of our main results on equilibria classification for a given participation $\alpha_0$ and model parameters $\Theta_0$. These results hold for dynamics which are risk minimizing in the limit and loss functions that are convex.}\label{fig:summary}
\end{figure}
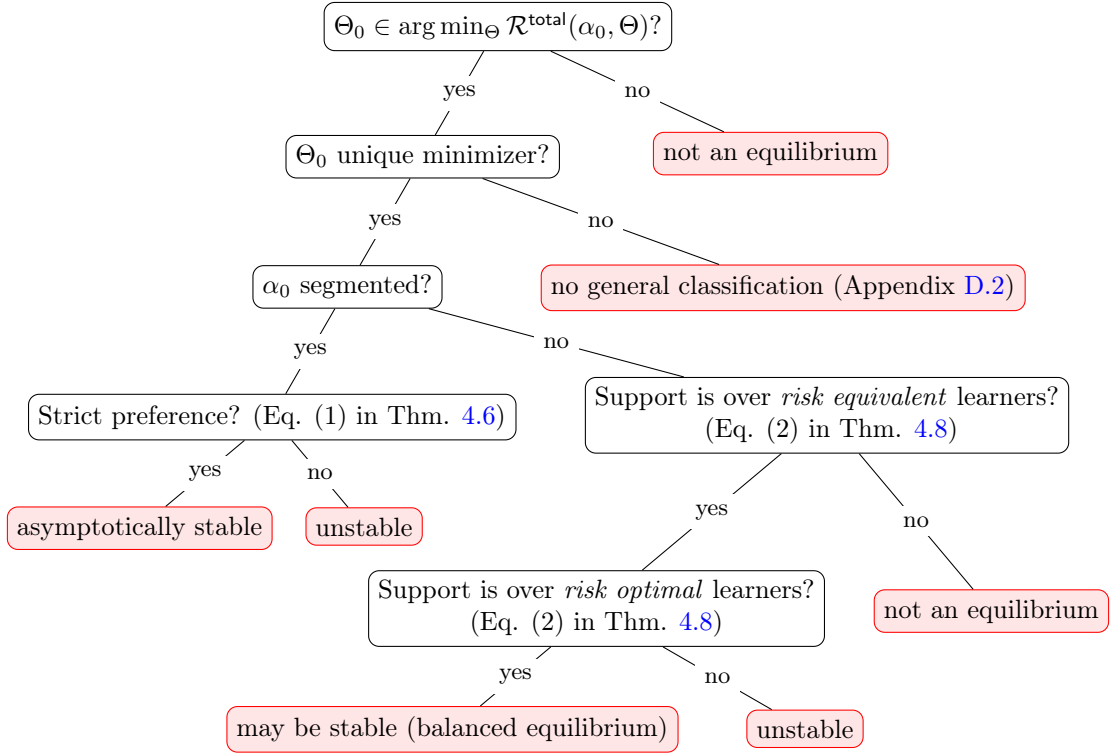
 \else
\begin{figure}[h]
\centering
\begin{tikzpicture}[
    node/.style={%
      draw,
      rectangle,
      rounded corners,
      align=center,
      font=\small,
    },
    leaf/.style={
      draw=red,
      fill=red!10,
      rounded corners,
      align=center,
      font=\small
    },
    decision/.style={
      midway,
      fill=white,
      font=\footnotesize,
    }
  ]
    \node [node] (A1) {$\Theta_0\in\arg\min_\Theta \risk^\mathsf{total}(\alpha_0,\Theta)$?};
    \path (A1) ++(-120:20mm) node [node] (A2) {$\Theta_0$ unique minimizer?};
    \path (A2) ++(-120:20mm) node [node] (A) {$\alpha_0$ segmented?};
    \path (A1) ++(-25:40mm) node [leaf] (A3) {not an equilibrium};
    \path (A2) ++(-20:51mm) node [leaf] (A4) {no general classification (Appendix \ref{sec:stability ex})};
    \path (A) ++(-120:20mm) node [node] (B) {Strict preference? (Eq. (1) in Thm. \ref{thm:segmented_market_eq})};
    \path (B) ++(-140:23mm) node [leaf] (F) {asymptotically stable};
    \path (B) ++(-50:19mm) node [leaf] (I) {unstable};
    \path (A) ++(-15:66mm) node [node] (C) {Support is over \emph{risk equivalent} learners?\\\ (Eq. (2) in Thm. \ref{thm:balanced_eq})};
    \path (C) ++(-140:40mm) node [node] (D) {Support is over \emph{risk optimal} learners?\\ (Eq. (2) in Thm. \ref{thm:balanced_eq})};
    \path (C) ++(-50:34mm) node [leaf] (E) {not an equilibrium};
    \path (D) ++(-30:32mm) node [leaf] (G) {unstable};
    \path (D) ++(-140:25mm) node [leaf] (H) {may be stable (balanced equilibrium)};

    \draw (A1) -- (A2) node [decision] {yes};
    \draw (A1) -- (A3) node [decision] {no};
    \draw (A2) -- (A) node [decision] {yes};
    \draw (A2) -- (A4) node [decision] {no};
    \draw (A) -- (B) node [decision] {yes};
    \draw (A) -- (C) node [decision] {no};
    \draw (B) -- (F) node [decision] {yes};
    \draw (B) -- (I) node [decision] {no};
    \draw (C) -- (D) node [decision] {yes};
    \draw (C) -- (E) node [decision] {no};
    \draw (D) -- (H) node [decision] {yes};
    \draw (D) -- (G) node [decision] {no};
\end{tikzpicture}
\caption{A summary of our main results on equilibria classification for a given participation $\alpha_0$ and model parameters $\Theta_0$. These results hold for dynamics which are risk minimizing in the limit and loss functions that are convex.}\label{fig:summary}
\end{figure}
 \fi

\subsection{Total Risk Reduction}

\begin{definition}[Total Risk]
The \emph{total risk} of all \groups{} over all \players{} is %
the weighted sum \[\risk^{\total}(\alpha, \alltheta) :=\sum_{\indgroup=1}^\numgroup\sum_{\indplayer=1}^{\numplayer}\groupsize_\indgroup \partic_{\indgroup\indplayer}\risk_\indgroup(\theta_\indplayer).\]
\end{definition}

The total risk maps $\Delta_\numplayer^\numgroup \times \R^{\numplayer\times d}\to\R$.
While our assumption that the loss is convex implies that the total risk is convex in $\Theta$, it is not jointly convex in $(\alpha, \alltheta)$, illustrated in the right panel of Figure~\ref{fig:lin_2d_ex}.

Our first result shows that the total risk $\risk^{\total}(\alpha^t, \alltheta^t)$ is non-increasing over time.

\begin{restatable}{prop}{totalRiskDec}
\label{prop:total_risk_dec}
For any risk-reducing \group{} and \player{} dynamics, the total risk is non-increasing:
$\risk^{\total}(\alpha^{t+1}, \alltheta^{t+1}) \leq \risk^{\total}(\alpha^t, \alltheta^t), \forall t$.
If \groups{}  and \players{} are risk minimizing in the limit, then the total risk is strictly decreasing unless $(\alpha^t, \alltheta^t)$ is a local minimizer of $\risk^{\total}$. %
\end{restatable}

\begin{proof}[Proof Sketch]
First note that the total risk can be decomposed into either a weighted sum of average \group{} risk or average \player{} risk. Thus the fact that \player{} and \group{} dynamics are risk reducing ensures that the total risk is decreasing after the sequential updates.
\end{proof}

Thus, the total risk acts like a potential function for the feedback dynamics of \players{} and \groups{}.
When the \group{} and \player{} dynamics are risk minimizing in the limit, 
there is a strong connection between properties of the total risk function and equilibria of the dynamics.

\begin{restatable}{theorem}{eqMinConnection}
\label{prop:eq_min_connection}
For any \players{} and \groups{} who are risk minimizing in the limit, an equilibrium $(\partic^\eq, \alltheta^\eq)$ is asymptotically stable if 
it is an isolated local minimizer of the total risk $\risk^{\total}$. {If it is not a local minimizer of the total risk, then it is not stable.}
\end{restatable}
\begin{proof}[Proof Sketch] By Proposition \ref{prop:total_risk_dec} the function $V(\partic, \alltheta):= \risk^{\total}(\partic, \alltheta) - \risk^{\total}(\partic^\eq, \alltheta^{\eq})$ is potential function for the autonomous dynamical system $(\partic^t, \alltheta^t) \to (\partic^{t+1}, \alltheta^{t+1})$. The stability result follow from Lyapunov arguments.
\end{proof}

The connection between stability and the total risk function is significant in at least two ways: first, it means that under general classes of myopic and self-interested behaviors on the part of \groups{} and \players{}, the total risk is driven to at least a local minimum.
Second, it is a technically useful connection that will enable us to characterize and classify the stable equilibria for dynamics which are risk minimizing in the limit.
We remark that Theorem~\ref{prop:eq_min_connection} leaves open the question of stability for equilibria which are non-isolated minima of the total risk function.
In  Appendix~\ref{sec:stability ex}, we provide examples which show that such points may be asymptotically stable, stable, or unstable depending on the particular instantiation of dynamics.
The following existence result further motivates our focus dynamics which are risk minimizing, rather than just reducing.

\begin{restatable}{coro}{eqExistence}\label{coro:existence}
    Equilibria exist when learners and subpopulations are risk minimizing in the limit and the total risk function has %
    isolated local minima. They may not exist otherwise.
\end{restatable}
\begin{proof}[Example of dynamics without equilibria]
Consider subpopulations with risk functions minimized at the same value $\theta^*$. If learners use full risk minimization, the setting lacks isolated minima because the total risk is uniform across all allocations $\alpha$. Assuming that risk-minimizing subpopulations randomly choose among equivalent learners, no equilibrium exists as allocations randomly switch between learners once the learners converge to the optimum $\theta^*$.
\end{proof}

\subsection{Segmented and Balanced Equilibria}

\begin{definition}[Segmented allocation]
An allocation is \emph{segmented} if $\partic_{\indgroup\indplayer}\in \{0,1\}$ for all $\indgroup,\indplayer$.
\end{definition}

In a segmented allocation, each \group{} is associated with a single \player{}, and thus the population is partitioned across learners.
For allocation dynamics like multiplicative weights, such configurations are clearly equilibria for any parameter choice $\Theta$ on the part of the \players{}. 
We thus consider the set of possible segmented equilibria and characterize which are asymptotically stable.

\begin{restatable}{theorem}{segmentedMarketEq}
\label{thm:segmented_market_eq}
Suppose \players{} and \groups{} are risk minimizing in the limit, $\alpha^\eq$ is segmented, and $\risk^\total(\alpha^\eq,\alltheta)$ {has a unique minimizer $\alltheta^\eq$}.
Define a mapping $\gamma:[\numgroup]\to[\numplayer]$ such that $\gamma(\indgroup)=\indplayer$ is the \player{} with nonzero mass in $\partic^\eq_{i,:}$.
If 
every \group{} strictly prefers their current \player{}:
\begin{align} \label{eq:nec_as}
\risk_\indgroup(\theta^\eq_{\gamma(\indgroup)}) < \risk_\indgroup(\theta^\eq_{\indplayer})\:,
\end{align}
for all  $\indgroup$ and \players{} $\indplayer\neq \gamma(\indgroup)$, 
then
$(\alpha^\eq,\alltheta^{\eq})$ is an asymptotically stable equilibrium.
If 
there is a \group{} who would strictly prefer to switch \players{},
then $(\alpha^\eq,\alltheta^{\eq})$ is not stable.
\end{restatable}
When risks are strongly convex, there is always such a unique minimizer $\alltheta^\eq$.
In particular, in a segmented allocation, each $\theta_j^\eq$ minimizes the average loss over the group of \groups{} assigned to them.

\begin{restatable}{coro}{convexHull}
\label{coro:convex_hull}
{Suppose that risk functions satisfy $\risk_\indgroup(\theta) < \risk_\indgroup(\theta') \Longleftrightarrow \Vert \theta - \phi_{\indgroup} \Vert < \Vert \theta' - \phi_{\indgroup} \Vert$ for $\phi_\indgroup$ the \group{} optimal parameter.}
Then in an asymptotically stable segmented equilibrium, the convex hulls of the grouped \groups{} optimal parameters $\{\phi_\indgroup\}$ are non-intersecting.
\end{restatable}

\begin{proof}[Proof Sketch]
Consider a partition where the convex hulls intersect for some pair of \players{}.
Then there exists at least one \group{} who would be better off switching to the other \player{}, and thus the risk condition in Theorem~\ref{thm:segmented_market_eq} cannot hold.
\end{proof}
Applying the Corollary to the example in Figure~\ref{fig:lin_2d_ex}, we see that a segmented equilibrium with \group{} 1 and 3 participating in the same learner cannot be stable. 
Theorem~\ref{thm:segmented_market_eq} leaves open the question of stability in the case
that the risks in Equation~\eqref{eq:nec_as} are equal.
Under such \emph{risk equivalence}, is it natural to consider
equilibria where a \group{} has support over multiple \players{}.

\begin{restatable}{theorem}{balancedEq}
\label{thm:balanced_eq}
Consider dynamics which are risk minimizing in the limit and an $\alpha^\eq$ with any \group{} $\indgroup$ having nonzero support on set of two or more \players{} $\indplayer\in\mathcal J$. \propchange{Assume risks are strongly convex} and define $\alltheta^\eq = \arg\min\risk^\total(\alpha^\eq,\alltheta)$.
Then $(\partic^\eq,\alltheta^\eq)$ {cannot be stable unless} it is ``balanced'' in the sense that \players{} in $\mathcal J$ are \emph{risk equivalent} and \emph{optimal} for $\indgroup$, i.e.
for all $j,j'\in\mathcal J$, 
\begin{equation}\risk_\indgroup(\theta^\eq_\indplayer)=\risk_\indgroup(\theta^\eq_{\indplayer'})\quad\text{and}\quad \nabla \risk_\indgroup(\theta_\indplayer^\eq)=0\:.%
\end{equation}
{If it is balanced, so are} all allocations for \group{} $\indgroup$ with support over $\mathcal J$.
Finally, all stable equilibria must be either balanced or segmented.
\end{restatable}

This result characterizes a set of possibly stable equilibria.
It demonstrates that risk \emph{optimality}, in addition to \emph{equivalence}, is necessary.
Guaranteeing the stability of such balanced equilibria requires further information about the dynamics, and  it is not possible to make a general statement.
Examples in Appendix~\ref{sec:stability ex} demonstrate that such balanced equilibria may be asymptotically stable, stable, or unstable.
Furthermore, the balance condition is fragile in the sense that it would not hold under small perturbations to the underlying risk functions.
While the number of possible balanced equilibria is combinatorial in the number of \players{} and \groups{}, risk functions are continuous,
so it is possible to find arbitrarily small perturbations to any the risk functions that would destabilize all balanced equilibria. %

\begin{proof}[Proof Sketch of Theorems~\ref{thm:segmented_market_eq} and~\ref{thm:balanced_eq}]
By Theorem~\ref{prop:eq_min_connection}, characterizing the stable equilibria is equivalent to characterizing isolated and non-isolated local minima of the total risk.
We show that it suffices to characterize local minima of the partial minimization $F(\alpha) = \min_\alltheta \risk^\total(\alpha,\alltheta)$ over the simplex product $\Delta_m^n = \Delta_\numplayer\times\dots\times\Delta_\numplayer$.
Since $F$ is concave, all minima occur on the boundary, i.e. a face or a vertex.
Since $F$ is still concave when restricted to a face of the simplex, the same argument shows the minima are on the boundary, hence vertices, except for the degenerate case where $F$ takes a constant value over the face.

Thus, the isolated local minima occur at vertices of the simplex product, which correspond to segmented allocation.
Further analysis of $F$ yields the conditions presented in Theorem~\ref{thm:segmented_market_eq}.
The local minima in the degenerate case are characterized by the balanced equilibria conditions in Theorem~\ref{thm:balanced_eq}.
\end{proof}

\begin{figure*}[h]
\centering
\begin{subfigure}{0.55\linewidth}
        \vskip 0pt
        \hspace*{-0.5cm}    
        \centering
        \includegraphics[height = 3.7cm]{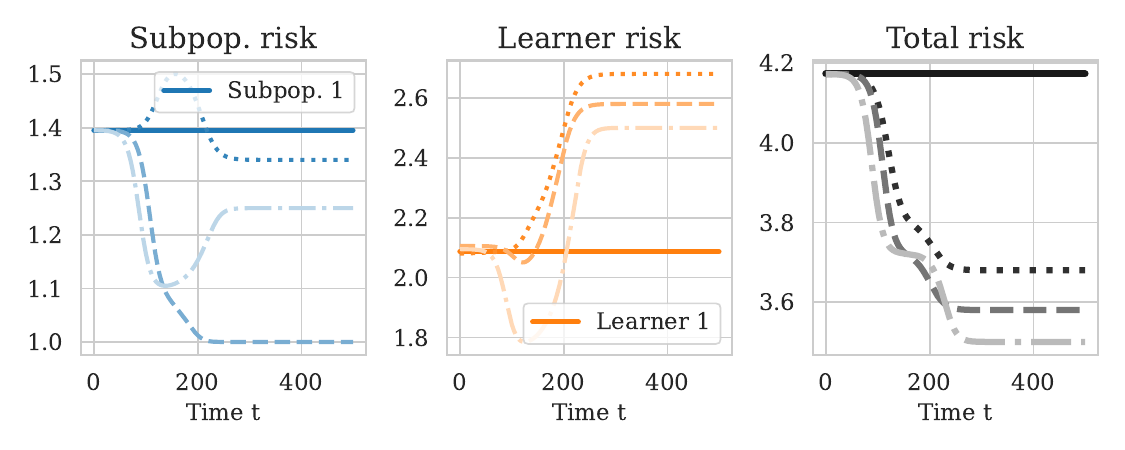}
    \caption{{Risk dynamics}}
    \label{fig:cost_curves}
\end{subfigure}
~
\begin{subfigure}{0.42\linewidth}
        \vskip 0pt
        \centering
    \includegraphics[height = 3.7cm]{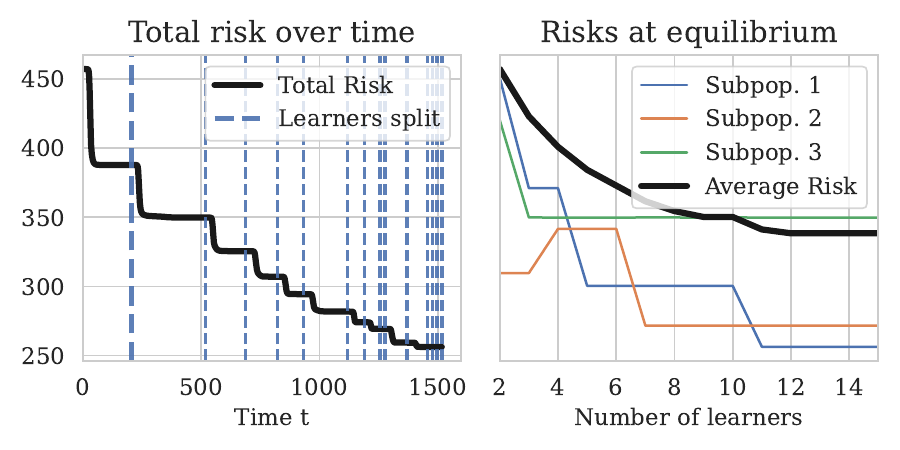}
    \caption{{Impact of competition on social welfare}
    }
    \label{fig:competition}
\end{subfigure}
\caption{{\textbf{Synthetic settings:} 
    Figure (a) illustrates a setting with 3 \groups{} and 2 \players{}. The dsolid lines correspond to the %
    risk trajectory for the unstable balanced equilibrium at initialization. Dotted and dashed lines illustrate risk trajectories under three different slight perturbations from the initialization.
    In Figure (b), the left plot illustrates the reduction in total risk over time. The dashed blue lines indicate when a new \player{} joins. The right plot shows the equilibrium-risk for a subset of the \groups{} as the number of \players{} increases.}
    }
    \label{fig:cost_curves_competition}
\end{figure*}

\subsection{Social Welfare for Segmented Populations}
\begin{definition} %
The \emph{social welfare} of a state $(\partic,\alltheta)$ \propchange{is strictly {decreasing} } in the total risk  $\risk^{\total}(\alpha, \alltheta)$.
\end{definition}
This definition of social welfare is utilitarian in the sense that
it depends on the cumulative quality of individuals' experiences.
Maximizing the social welfare corresponds to minimizing the total risk, which can be posed as the following optimization problem
\begin{align}\begin{split}\label{eq:TR_opt}
    (\partic^\star,\alltheta^\star) \in &\arg\min_{\partic,\alltheta}
    \quad \risk^{\total}(\alpha, \alltheta) 
    \\
    &\text{s.t.} \quad  \alpha_{\indgroup,:} \in\Delta_{\numplayer}~~\forall~~\indgroup=\{1,\dots \numgroup\}\:.\end{split}
\end{align}
Here, $(\partic^\star,\alltheta^\star)$ is the social welfare maximizer.

Our discussion of stable equilibria has so far focused on only local minimizers of the total risk.
In fact, global minimization of this objective (and therefore maximization of social welfare) is a hard problem.
The total risk objective can be viewed as an instance of the $k$-means clustering problem with $k=\numplayer$. 
In the language of this literature (e.g.,  \cite{selim1984k}), each \group{} is a data point %
and the parameter selected by each \player{} is a cluster center.
The allocations described by $\partic$ correspond to (fuzzy) cluster assignment and each risk function $\risk_\indgroup(\theta_\indplayer)$ corresponds to a measure of ``dissimilarity'' between %
data points (\groups{}) and cluster centers (\players{}).

The connection to $k$-means clustering elucidates the difficulty of minimizing the total risk.
The ``minimum sum-of-squares clustering'' problem (i.e., squared Euclidean norm dissimilarity) is NP hard with general dimension even when $k=2$~\citep{aloise2009np}.
When the number of clusters and dimension are fixed,
\citet{inaba1994applications} present an algorithm for solving the minimum sum-of-squares clustering problem which is polynomial in the number of datapoints. Translated to our setting, its complexity is $O(\numgroup^{\numplayer d})$.
It is therefore unrealistic to hope that a myopic dynamic might generally lead to social welfare maximization. 
However, due to the connections with total risk, risk reducing dynamics are at least well-behaved with regards to social welfare. 

\begin{prop}\label{prop:socialwelf}
For risk reducing \groups{} and \players{}, 
social welfare is non-decreasing over time.
If the dynamics are furthermore risk
minimizing in the limit, social welfare is strictly increasing and stable equilibria correspond to local social welfare maxima.
\end{prop}
Local maximization is not a panacea: Example~\ref{ex:local-global} shows  a local maximum of the social welfare can be much worse than the global one.

\begin{example}[Arbitrarily high total risk at local optimum]\label{ex:local-global}
Consider three \groups{} with
\[\risk_1(\theta) = \theta^2,\quad \risk_2(\theta) = (\theta-1)^2,\quad \risk_2(\theta) = (\theta-\phi)^2\]
for some $\phi>2$.
Suppose that \group{} sizes are $\beta_1=\beta_2=\beta$ and $\beta_3=1-2\beta$ for some $0<\beta<1/2$.
Further suppose that there are two \players{}. Up to permutation, the social welfare optimum is $\theta_1=1/2$ and $\theta_2=\phi$, with total risk $\beta/2$.
However, as long as $\phi<\frac{1-\beta}{1-2\beta}$, there is another stable equilibrium. Let $\phi=\frac{1-\beta}{1-2\beta} - \epsilon$. Then the following is a stable equilibrium: $\theta_1=0$ and $\theta_2=1-\epsilon$. The total risk is $\beta+\frac{(\beta-\epsilon)^2}{1-2\beta}$. For $\beta$ close to $1/2$, this risk can be arbitrarily larger than the social optimum.
\end{example}

In this example, a large gap between a stable local optimum and the global optimum arises in part due to a large difference in \groups{}' sizes.
We further remark that minority groups can be under-served particularly when considering worst-case risk over \groups{}~\citep{hashimoto2018fairness}.
Even at a social welfare maximizer $(\partic^\star,\alltheta^\star)$, the worst-case \group{} risk can be arbitrarily bad.
It is straightforward to construct such examples even in the single \player{} case:
consider a minority group with vanishingly small population proportion and arbitrarily high risk at the optimal parameter for the majority group (Example~\ref{ex:minority-worst-case}).

Despite these inherent difficulties, we find that the situation improves as the number of \players{} increases. %
It is straightforward to see that the maximal social welfare will increase:
any %
point which is optimal for $\numplayer$ \players{} can be trivially transformed into a feasible point for $\numplayer+1$ \players{} which achieves the same social welfare, by allocating no \groups{} to the new \player{}.
There is more nuance involved when considering any possible stable equilibria.
Instead, we make a statement about a particular \player{} growth process which corresponds to existing \player{} $\numplayer$ ``splitting in half''.
\begin{prop}\label{prop:add_player}
Suppose that \propchange{risks are strongly convex}, there are $\numplayer$ \players{}, $(\partic^{\eq},\alltheta^{\eq})$ is an equilibrium, and at least one \group{} $\indgroup$ allocated to \player{} $\numplayer$ does not have optimal \group{} risk, so $ \nabla \risk_\indgroup(\theta_\numplayer^\eq)\neq0$.
The state is amended to add an additional \player{}: $\tilde\alltheta^\eq = [\alltheta^\eq, \theta^\eq_\numplayer]$ and
\[\tilde\partic_{:,j}^\eq =\begin{cases} \partic^\eq_{:,j}&j\leq  \numplayer\\
\frac{1}{2}\partic^\eq_{:,\numplayer} &  j\in\{\numplayer,\numplayer+1\}
\end{cases}\]
Under dynamics which are risk minimizing in the limit, the equilibrium $(\tilde\partic^\eq, \tilde\alltheta^\eq)$ is not  stable, so a small perturbation will send the system to a state with strictly lower total risk (higher social welfare). %
\end{prop}

\ifsub
    \section{{SIMULATIONS}}\label{sec:experiments}
\else
    \section{{Simulations}}\label{sec:experiments}
\fi

We illustrate the salient properties of the decision dynamics in simulation\footnote{Implementation details and reproduction instructions at:\\ \url{https://github.com/mcurmei627/MultiLearnerRiskReduction}}.
We consider both a synthetic setting as well as one instantiated from a prediction task on census data.

\textbf{Synthetic}\quad
In Figure \ref{fig:cost_curves} we consider a simple scenario with $n=3$ \groups{} of equal sizes $\beta_{\indgroup} = 1/3$, quadratic risk functions $\risk_{\indgroup} = \Vert \phi_\indgroup - \theta\Vert^2 + 1$ with distinct risk minimizing decisions $\phi_{\indgroup}$ and $m=2$ \players{}. 
The \players{} minimize their risk according to \emph{full risk minimization} (Example \ref{ex:full_player}) and the \groups{} update their participation via multiplicative weights update (Section~\ref{sec:dynamics}). When $\partic^{0}_{\indgroup, \indplayer} = 1/2$ for all $\indgroup, \indplayer$ 
the risk equality condition from Theorem \ref{thm:balanced_eq} is satisfied 
with $\theta^{\eq}_{\indplayer} = (\phi_1 + \phi_2 + \phi_3)/3$, however the optimality condition is not. We therefore observe that this equilibrium is not stable, and slightly perturbing the initial conditions leads to split-market equilibria. Figure \ref{fig:cost_curves} illustrates trajectories from three different perturbations. It demonstrates that the total risk is non-increasing whereas the average risks for both \players{} and \groups{} are not monotonic. Each of the perturbations has different risk trajectories and equilibrates at a different split-market equilibrium. We repeat these experiments with noisy dynamics, we consider both exogenous noise that independently perturbs the decisions of the learners and/or populations as well as intrinsic noise due to making updates with finite sample estimates rather than at population level. We find that the key properties of the dynamics hold when the updates are noisy, detailed experiments are presented in Appendix \ref{app:noisy_dynamics}.

Another set of experiments in Figure \ref{fig:competition} illustrates how a larger number of \players{} lead to better outcomes in terms of total risk. We consider a set of $\numplayer{} = 2$ \players{} and $\numgroup{} = 50$ \groups{}. We simulate the dynamics until the market has reached equilibrium, at which point a randomly chosen \player{} breaks up into two identical \players{} with half the user base. 
From this unstable equilibrium (Proposition \ref{prop:add_player}) we slightly perturb the parameters of the two \players{} and allow the system to reach a new equilibrium state.
The procedure repeats until the number of \players{} reaches number of \groups{}. These simulations illustrate that more competition improves social welfare, however the improvements are not uniform for all \groups{} with some groups seeing their risk at equilibrium increase with the addition of new \players{}.

\begin{figure}[h]
    \centering
    \begin{subfigure}[b]{0.56\textwidth}
        \centering
        \includegraphics[width=\textwidth]{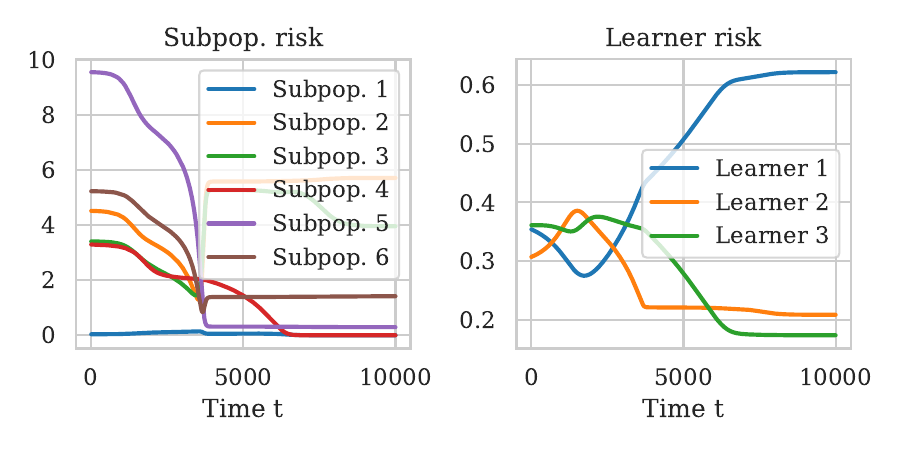}
        \caption{Risk dynamics}
        \label{fig:census_risk}
    \end{subfigure}%
    ~ %
    \begin{subfigure}[b]{0.4\textwidth}
        \centering
        \includegraphics[width=\textwidth]{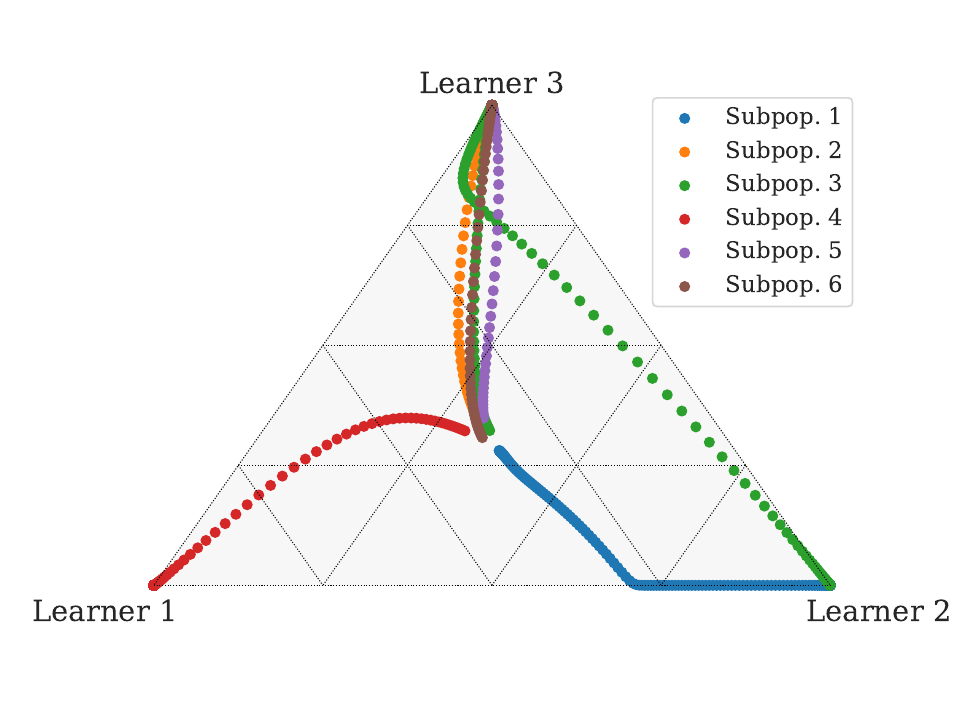}
        \caption{Allocation trajectories}
        \label{fig:census_ternary}
    \end{subfigure}
    \caption{{\textbf{Empirical \groups{} from Census data:} 
    Figure (a) displays the relative risk with respect to the best achievable risk for the \group{} over time. 
    Figure (b) illustrates how allocations initialized near $(1/3,1/3, 1/3)$ converge to a split market equilibrium.}}
    \label{fig:census}
\end{figure}

\textbf{Census data} \quad We consider a semi-synthetic setting where \groups{} and their risk functions are instantiated by a prediction task on real data. Using \texttt{folktables} \citep{ding2021retiring} we consider  a modified version of \emph{ACSTravelTime} prediction problem derived from the 2018 California Census data. We consider 6 \groups{} corresponding to racial groups with relative size ranging from 1.2\% to 61\%. We define the least-squares risk functions as $\risk_{\indgroup}(\theta) = \frac{1}{N_\indgroup}\Vert X_\indgroup\theta - y_\indgroup\Vert^2$ where $X_\indgroup\in\R^{N_i\times d}$ are the features (containing demographics, educational attainment, income levels, and modes of transportation) and $y_\indgroup\in\R^{N_i}$ are the labels (log transform of the daily commute time in minutes) for individuals within \group{} $\indgroup$.
We simulate risk reducing dynamics from a perturbed balanced equilibrium over 3 \players{}. As in the synthetic example, the risks of \players{} and \groups{} are not all monotone (Figure \ref{fig:census_risk}), but the total risk function is. Finally Figure \ref{fig:census_ternary} illustrates the convergence of allocation dynamics to a segmented equilibrium.

\ifsub
   \section{DISCUSSION}
\label{sec:conclusion}
\label{sec:relatedwork}
\else
   \section{Discussion}
\label{sec:conclusion}
\fi

In this paper, we study the feedback dynamics of user retention for loss 
minimizing learners, where \groups{} choose between  providers.
We introduce a formal notion of \emph{risk reducing} and \emph{minimzing} to capture this feedback, and 
show that there is a close connection between such dynamics and the \emph{total risk} summed over \groups{} and \players{}.
We provide a comprehensive characterization of stable equilibria and investigate the  implications in terms of a utilitarian social welfare.
This work relates to questions of fairness and minority representation in several ways.
First, our results imply that risk-minimizing dynamics 
in multi-learner settings can result in higher welfare for small \groups{}
 compared with single-learner settings, as studied by~\citep{hashimoto2018fairness,zhang2019group}. This resonates with recent work showing that monopolies have higher  \emph{performative power} and lead to lower individual utility~\citep{hardt2022performative}. 
 
The dynamics that we study often lead to \emph{segmentation} of \groups{} across \players{} as an emergent phenomenon\footnote{This connects to economic literature on ``rational" discrimination, where competitors have no inherent preference to discriminate and yet equilibria are segregated, e.g.~\citet{foster1992economic}}. 
This segmentation can lead to pointwise lower risks for \groups{}, especially when \groups{} have considerably different risk profiles.
In some contexts, the benefits of the reduced risk among \groups{} may outweigh possible harms from segregation. In others, where proportional representation of groups across \players{}, models, or clusters~\citep{kleindessner2019fair,kleindessner2019guarantees} is important, our work implies that independent risk minimization can lead to undesirable outcomes. In short, this work analyzes natural dynamics with consequences for the distribution of \groups{} amongst independent \players{}; whether or not the consequences are desirable depend on the specific application considered.

We highlight several directions for future work.
Our results lay the groundwork for an investigation of the stochastic dynamics that occur for finite sample approximations to the risk or participation driven by decisions of individuals.
Such behaviors are risk reducing in \emph{expectation}, so we expect the noisy trajectories to converge with high probability to sets around the asymptotically stable equilbria we characterize. 
There are many interesting and relevant questions in the finite sample setting: 
What is the effect of sample size on the ability of new \players{} to enter a market and minority \groups{} to be adequately represented?
Can we model heterogeneous \players{} who differ in which features they measure and with how much noise?
Are there trade-offs between the expressivity of models and the practical difficulty of minimizing risk from finite samples in high dimensions?

It would also be interesting to consider extensions or alternative dynamics models for the \player{} and \group{} decisions.
One could investigate competitive \players{} who explicitly strategize to capture \groups{} \citep{ben2019regression,aridor2020competing}; this setting is related to facility location and Hotelling games \citep{owen1998strategic,hotelling1929stability}.
One might imagine that \groups{} do not act uniformly and may not even be entirely independent of each other---the participation update may depend on some underlying social network.
The connections between total risk reduction and
$k$-means clustering algorithms
suggest interventions such as \group aware initialization \citep{bose2023initializing} that could improve social welfare. 
Results on ``ground truth recovery'' 
may yield insight into particular population structures that lead to simpler dynamics or restricted sets of equilibria.

\pagebreak 

\section*{Acknowledgements}
We thank Laurent Lessard for suggesting the clever argument to prove Lemma~\ref{lem:chull}.
JM is supported by an NSF CAREER award (ID 2045402), an NSF AI Center Award (The Institute for Foundations of Machine Learning), and the Simons Collaboration on the Theory of Algorithmic Fairness. LJR is supported by  NSF CNS-1844729, NSF IIS-1907907, and Office of Naval Research YIP Award N000142012571. 
MF is supported by NSF TRIPODS II-DMS 2023166, NSF TRIPODS-CCF 1740551, and NSF CCF 2007036.

\bibliographystyle{plainnat}
\bibliography{refs.bib}

\newpage
\appendix

\section{Motivating Examples}\label{app:motivation}

We discuss several real-world examples which exhibit
degrees of market segmentation across characteristics such as nationality, age, and race. 
In these examples, market conditions are certainly affected by more complex phenomena, from network effects to explicit competition between firms.
While we do not claim that the dynamics we study are necessarily the main contributing factor, 
our simple model isolates the potential contribution of 
learning dynamics: namely, to reinforce such segmentation.
This perspective highlights the potential effects of efforts to incorporate data or improve personalization.

\subsection{Social Media}
\begin{figure}[h]
    \centering
    \includegraphics[width=0.49\textwidth]{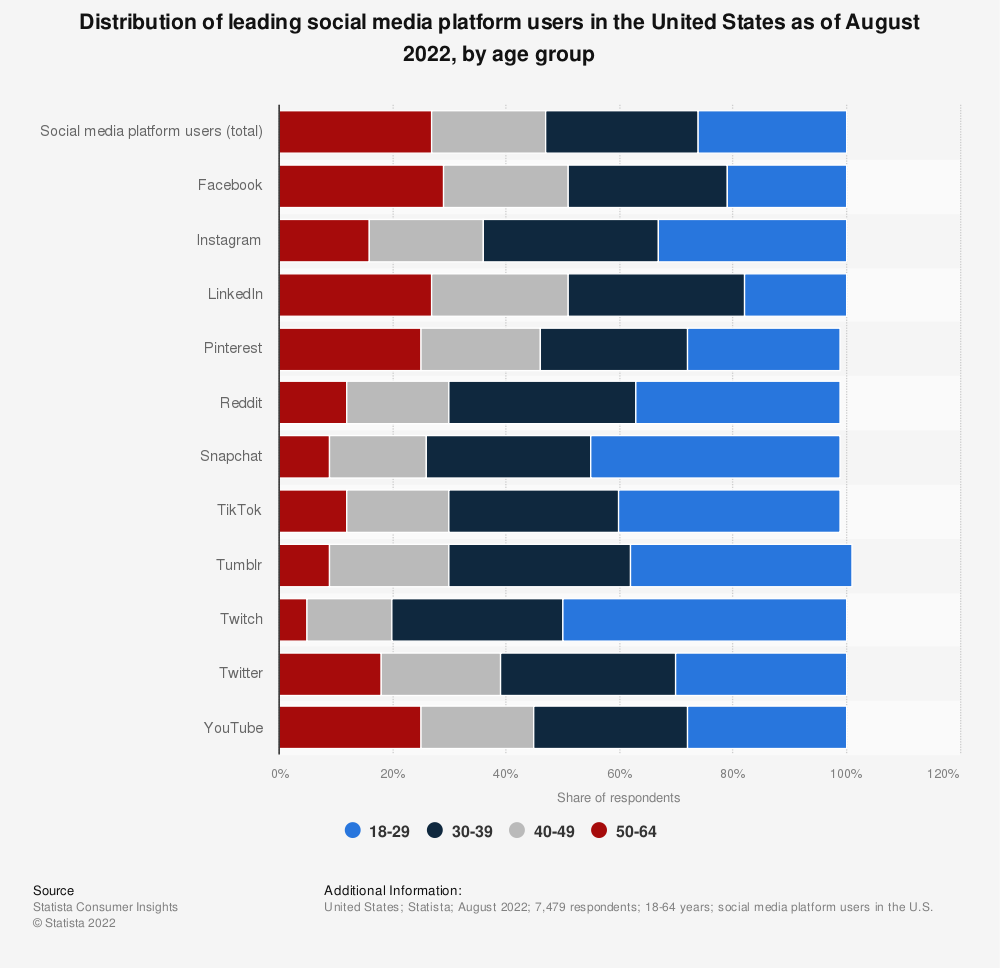}
    \includegraphics[width=0.49\textwidth]{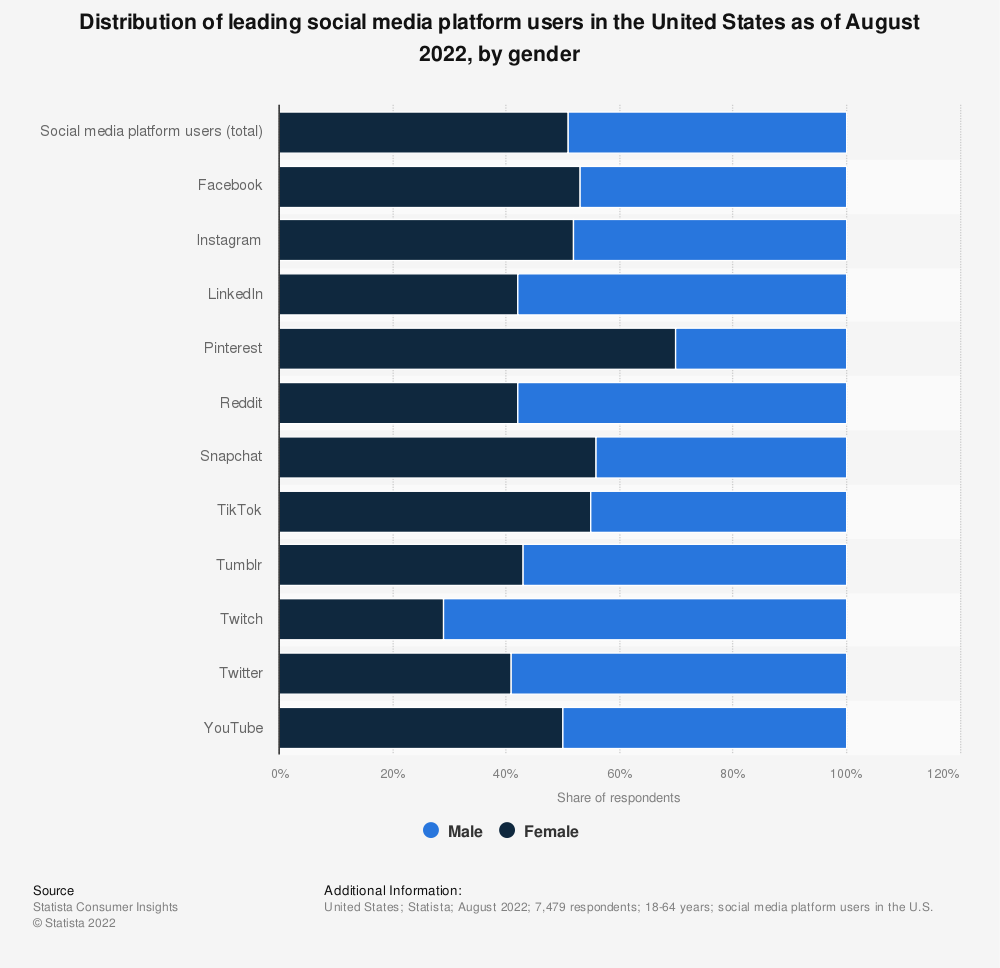}
    \caption{Social media usage across leading social media platforms. Left: Age distribution. Right: Gender distribution}
    \label{fig:social_media}
\end{figure}
Usage of various social media sites in the US varies across genders\footnote{\url{https://www.statista.com/statistics/1337563/us-distribution-leading-social-media-platforms-by-gender/}} and age groups\footnote{\url{https://www.statista.com/statistics/1337525/us-distribution-leading-social-media-platforms-by-age-group/}}. For example, the users of Facebook and LinkedIn skew older while Snapchat, Tiktok, Tumblr, and Twitch are more heavily used by the younger population. Similarly users of Pinterest strongly skew female while users of Twitch are more likely to be male. Figure \ref{fig:social_media} shows the disparities along gender and age for leading social media platforms. These disparities across platforms are reinforced by user behaviors: imaging the experience of a 45 year old logging onto Twitch for the first time compared with a 14 year old; or instead imagine a 14 year old logging into Facebook. Because the usage patterns determine the data available to the platforms, the disparities are also reinforced by the behavior of the platforms themselves. Similarly, Pinterest algorithms are more likely to be tailored to the tastes of an female demographic, while Twitch’s to a demographic more representative of males.

\subsection{Music Streaming}
Worldwide market share of music streaming services is split between several companies (see Figure~\ref{fig:music-market}). However, the distribution of music streaming by country shows clear patterns: most users in China use Tencent, most users in Mexico use Spotify, and most users in the Middle East and Northern Africa (MENA) use Anghami. On the other hand, the markets United States, Russia, and India are not dominated by a single service. 
However, the handful of most used services in these regions have a small market outside of their main market.
Due to this segmented market,
only certain platforms collect large scale data about music preferences in certain regions.
If many users from western cultures make playlists containing both Arabic and Indian music,
Spotify may learn to associate those genres in a way that is undesirable or even offensive to users from those cultures.
This leads to a self-reinforcing effect: services who make bad predictions for users from certain cultures are unlikely to correct this bias as those users choose instead to use services that more accurately reflect their tastes.
\begin{figure}[h]
    \centering
    \begin{subfigure}[b]{0.32\linewidth}
    \centering
    \includegraphics[height=5.5cm]{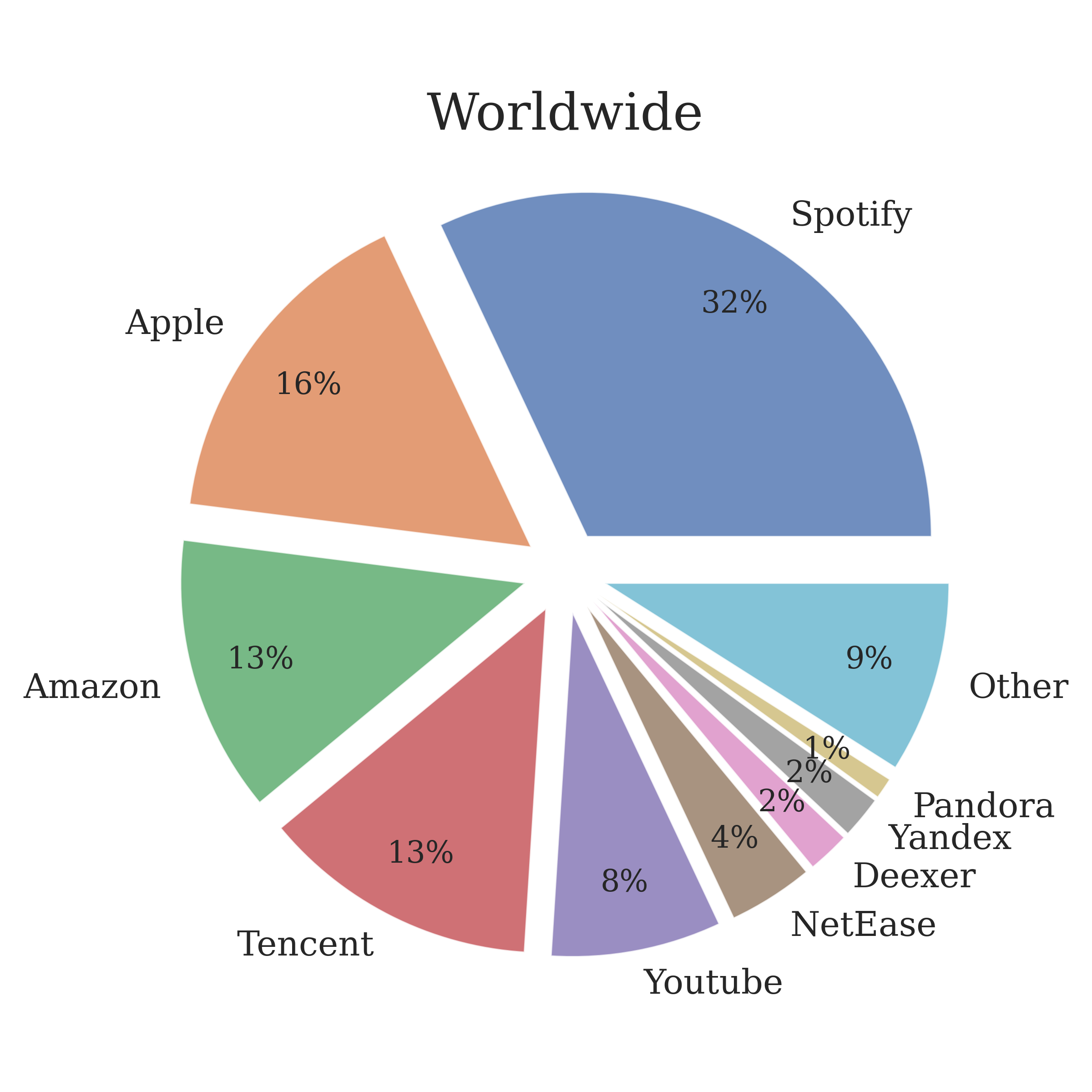}
    \caption{Worldwide usage}
    \label{fig:music_market_a}
    \end{subfigure}
    \begin{subfigure}[b]{0.64\linewidth}
    \includegraphics[height=5cm]{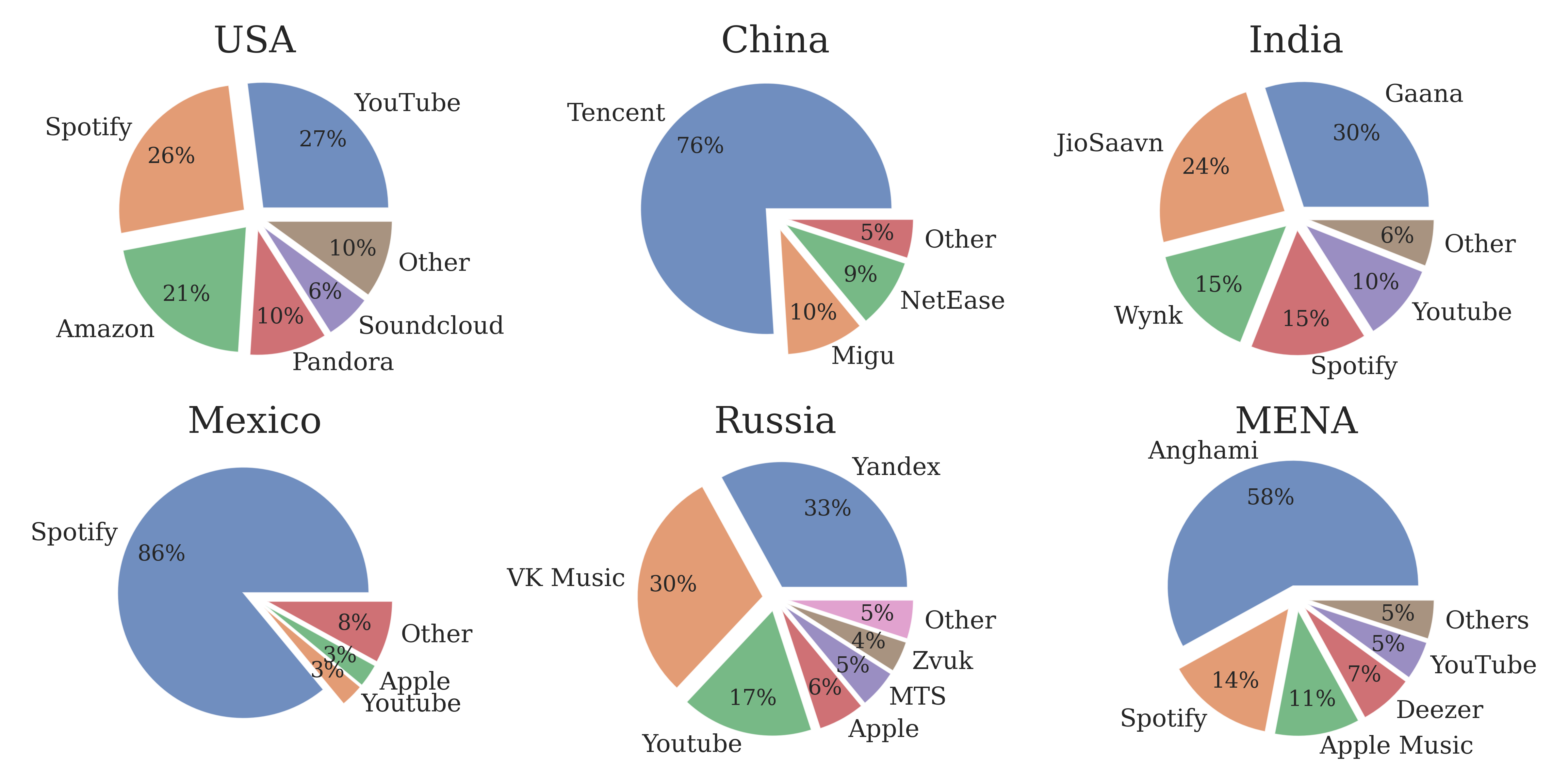}
    \caption{Usage by country}
    \label{fig:music_market_b}
    \end{subfigure}
    \caption{Usage of music streaming services in different markets\protect\footnotemark. Left: Worldwide marked share. Right: Market share in USA, China, India, Mexico, Russia, Middle East and Northern Africa (MENA).}
    \label{fig:music-market}
\end{figure}

\footnotetext{All statistics recorded from Statista: \\Worldwide:\url{https://www.statista.com/statistics/653926/music-streaming-service-subscriber-share},\\
United States:  \url{https://www.statista.com/statistics/1351506/streaming-services-music-podcasts-united-states/}, \\
China: \url{https://www.statista.com/statistics/711295/china-leading-mobile-music-platforms-by-active-user-number/}, \\
India: \url{https://www.statista.com/statistics/922400/india-music-app-market-share/},\\
Mexico: \url{https://www.statista.com/statistics/1018370/over-the-top-audio-platforms-mexico-by-market-share},\\
Russia: \url{https://www.statista.com/statistics/1347035/most-popular-music-streaming-platforms-in-russia/},\\
Middle East and North America: \\ \url{https://www.statista.com/statistics/1295716/mena-share-of-paying-music-streaming-subscribers-by-platform/}.
}

\subsection{Personalized health}
The growing popularity of direct-to-consumer genetic testing is driven by the growth of two market leaders: AncestryDNA and 23andMe\footnote{\url{https://www.statista.com/chart/17023/commercial-genetic-testing/}}. These tests are used both for determining ancestry as well as receiving polygenic risk scores for various medical conditions. The accuracy of the tests varies across ethnic groups; with Latino, Middle Eastern and, African ancestry being most under-represented. This issue is self re-inforcing; for instance people of African descent are less likely to use a large service like 23andMe and more likely to use a specialized service such as AfricanAncestry\footnote{\url{https://africanancestry.com/}}.  

\section{Preliminaries}

\subsection{Notation}

We introduce a compact notation. The simplex product is defined as
\[\Delta_\numplayer^{\numgroup} = \left\{A\in\R^{\numgroup\times\numplayer}\mid \sum_{\indplayer=1}^\numplayer A_{ij}=1\right\}\]
so that the \emph{rows} sum to 1.
Then the state space of \group{} allocations and \player{} parameters is
$\mathcal X = \Delta_\numplayer^{\numgroup}\times \R^{\numplayer\times d}$.
For a square matrix $A$, we use the notation $\diag(A)$ to represent the vector containing the diagonal entries of $A$. For a vector $a$, $\Diag(a)$ is a diagonal matrix with $a$ along the diagonal.
Furthermore we will say $a\leq b$ for vectors $a,b$ if the inequality holds elementwise.

Define a matrix valued  risk function $R:\R^{\numplayer\times d}\to \R^{\numgroup\times\numplayer}$ so that
$R_{ij}(\alltheta) = \risk_\indgroup(\theta_\indplayer)$.
Recall that in Section~\ref{sec:dynamics}, the \group{} and \player{} risks played a key role.
We therefore define vector valued functions $\bar\risk^{\groupm}:\mathcal X\to \R^\numgroup$ and $\bar\risk^{\playerm}:\mathcal X\to \R^\numplayer$ as follows: 
\[\bar\risk^{\groupm}(\partic,\alltheta) = \diag(\partic  R(\alltheta)^\top ),
\quad 
\bar \risk^{\playerm} (\alpha, \alltheta) = \diag\left(\Diag(\alpha^\top\beta)^{-1}\partic^\top \Diag(\groupsize) R(\alltheta)\right)\:.
\]
Then the definition of risk reducing dynamics for \groups{} and \players{} can be written as 
\[\bar\risk^{\groupm}(\partic^{t+1},\alltheta) \leq \bar\risk^{\groupm}(\partic^t,\alltheta) \quad \text{and}
\quad 
\bar \risk^{\playerm} (\alpha, \alltheta^{t+1}) \leq \bar \risk^{\playerm} (\alpha, \alltheta^t) \:.
\]
Risk minimizing in the limit is defined similarly, where the inequality is strict for at least one entry of the vectors unless the state is at a local minimum.

The total risk can be written as
\[\risk^{\total}(\alpha, \alltheta) := \tr( \diag(\groupsize) \partic  R(\alltheta)^\top )\:.\]

\begin{lemma}\label{lem:V_conts}
Under the assumption that all loss functions are continuous, the risk function $R$ is continuous w.r.t. to $\alltheta$, and thus $\risk^\total$ is continuous w.r.t. $\partic$ and $\alltheta$.
\end{lemma}

The sequential dynamics updates described in Section~\ref{sec:dynamics} can be written as 
\begin{align}\label{eq:def_f}
    \begin{bmatrix} \partic^{t+1} \\\alltheta^{t+1} 
    \end{bmatrix} &= \begin{bmatrix}\alloc(\partic^t,\alltheta^t)\\ \mu(\partic^{t+1}, \alltheta^{t})\end{bmatrix}= \begin{bmatrix}\alloc(\partic^t,\alltheta^t)\\ \mu(\alloc(\partic^t,\alltheta^t), \alltheta^{t})\end{bmatrix} =: f(\partic^{t},~\alltheta^{t})\:.
\end{align}

\begin{lemma}\label{lem:local_lipschitz}
As long as the \group{} and \player{} updates described in Section~\ref{sec:dynamics} are locally Lipschitz, so is the dynamics function $f$ defined in~\eqref{eq:def_f}.
\end{lemma}

\subsection{Background} 
For completeness, we include important results and definitions that our proofs will make use of.
First, we state two theorems about Lyapunov theory for stability.

\begin{theorem}[Theorem 1.2 in~\cite{bof2018lyapunov}]\label{thm:1.2}
Let $x_\eq\in\mathcal D$ be an equilibrium point for the autonomous systems $x_{t+1} = f(x_t)$ where $f:\mathcal D\to\mathcal X $ is locally Lipschitz in $\mathcal D\subseteq\mathcal X$.
Suppose there exists a function $V:\mathcal D\to\R$ which is continuous and such that
\begin{align*}
    &V(x_\eq)=0~~\text{and}~~V(x)>0~~\forall~~x\in\mathcal D-\{x_\eq\}\\
    &V(f(x)) - V(x) \leq 0 ~~\text{(resp.}~<0\text{)}~~\forall~~ x\in\mathcal D
\end{align*}
Then $x_\eq$ is stable (resp. asymptotically stable).
\end{theorem}

\begin{theorem}[Theorem 1.5 in~\cite{bof2018lyapunov}]\label{thm:1.5}
Let $x_\eq\in\mathcal D$ be an equilibrium point for the autonomous systems $x_{t+1} = f(x_t)$ where $f:\mathcal D\to\mathcal X $ is locally Lipschitz in $\mathcal D\subseteq\mathcal X$.
Let $V:\mathcal D\to\R$ be a continuous function with $V(x_\eq)=0$ and $V(x_0)>0$ for some $x_0$ arbitrarily close to $x_\eq$.
Let $r>0$ be such that $B_r(x_\eq)\subseteq \mathcal D$ and $\mathcal U = \{x\in B_r(x_\eq)\mid V(x)>0\}$, and suppose that $V(f(x))-V(x)>0$ for all $x\in \mathcal U$. Then $x_\eq$ is not stable.
\end{theorem}

Next, we state the definition of a (isolated) local minimum.

\begin{definition}
The point $u_\star$ is a local minimum (resp. isolated local minimum) of a function $h$ over a domain $\mathcal U$ if there is a $\delta>0$ such that for any $u\in\mathcal U$ with $\|u-u_\star\|\leq \delta$, $h(u_\star)\leq h(u)$ (resp. $h(u_\star)< h(u)$).
\end{definition}

Next, we state the implicit function theorem.

\begin{theorem}[Implicit Function Theorem]
Let $U\subseteq \mb{R}^n,\ V\subseteq\mb{R}^m$ be open sets and $f:U\times V\to \mb{R}$ is $C^r$ for some $r\geq 1$. For some $x_0 \in U,\ y_0 \in V$ assume the partial derivative in the second argument $D_2f(x_0,y_0) : \mb{R}^m \to \mb{R}$
 is an isomorphism. Then there are neighborhoods $U_0$ of $x_0$ and $W_0$ of $f(x_0,y_0)$ and a unique $C^r$ map $g:U_0\times W_0\to V$ such that for all $(x,w)\in U_0\times W_0$, $f(x,g(x,w))=w$.
\end{theorem}

Finally we prove a  property of intersecting convex hulls.

\begin{lemma}\label{lem:chull}
Let $x_1, x_2, \cdots x_n$ and $y_1, y_2, \ldots, y_m$ be some points in $\R^d$. Define by $\mc{C}_x$ and $\mc{C}_y$ the convex hulls of $\{x_i\}_{i=1}^n$ and $\{y_i\}_{i=1}^m$ respectively. Then there do not exist points $\bar x\in\R^d$ and $\bar y \in R^d$ such that the following inequalities are satisfied:
\begin{align*}
    \Vert x_i- \bar x\Vert &< \Vert x_i - \bar y\Vert\ \forall i=1, 2,\ldots, n\\
    \Vert y_j- \bar y\Vert &< \Vert y_j - \bar x \Vert\ \forall j=1, 2,\ldots, m
\end{align*}
\end{lemma}
\begin{proof}
Assume by contradiction that the inequalities above hold.
Define $\mc{H}_{x} :=\{z\in R^d~|~ \Vert z-\bar x\Vert <  \Vert z- \bar y\Vert\}$ and $\mc{H}_{y} :=\{z \in R^d~|~ \Vert z- \bar y\Vert  <  \Vert z- \bar x\Vert\}$. The sets $\mc{H}_x$  and $\mc{H}_y$ are disjoint half-spaces (without boundary) then defined by the hyperplane bisecting the segment connecting $\bar x$ and $\bar y$. By assumption then we have that $x_i \in \mc{H}_x$ for all $i$ and $y_j \in \mc{H}_y$  for all $j$; since $\mc{H}_x$ and $\mc{H}_y$ are convex, it follows that $\mc{C}_x\subset\mc{H}_x$ and $\mc{C}_y\subset\mc{H}_y$. Therefore $\mc{C}_x\cap \mc{C}_y = \emptyset$, which leads to a contradiction.
\end{proof}
\subsection{Properties of partial minimization}

In this section, we state a handful of important results about the partial minimization of the total risk.
This is somewhat similar to the analysis presented by~\citet{selim1984k} in the context of clustering algorithms.

\begin{restatable}{lemma}{localMinChar}
\label{lem:local_min_char}\label{lem:min_equiv}

Define the function $F:\R^{\numplayer\times\numgroup}\to\R$ as $F(\alpha) = \min_\alltheta \risk^\total(\alpha,\alltheta)$.
This function is concave and
a point $(\partic^0,\alltheta^0)$ is a local minimum of $\risk^\total$ over the domain $\mathcal X=\mathcal X_\partic\times \R^{\numplayer\times d}$ \propchange{if} $\partic^0$ is a local minimum of $F$ over the domain $\mathcal X_\partic$ and $\alltheta^0\in\arg\min_{\alltheta} \risk^\total(\partic^0,\alltheta)$.
\propchange{Furthermore, 
in the case that $\alltheta^0$ is the unique minimizer of $\risk^\total(\partic^0,\alltheta)$, then 
$(\partic^0,\alltheta^0)$ is a local minimum (resp. isolated local minimum) if and only $\partic^0$ is a local minimum (resp. isolated local minimum).}
\end{restatable}

\begin{proof}[Proof of Lemma~\ref{lem:local_min_char}]

$F(\alpha)$ is well defined due to the convexity of the risk functions.
Concavity follows from the observation that $F$ is the point-wise minimum of a family of functions which are linear in $\alpha$ (since for every fixed $\alltheta$, the total risk is linear in $\alpha$).

We break the proof of equivalence into two implications. 

1. $F$ minimized $\implies\risk^\total$ minimized\\
There is a $\delta>0$ such that for any $\partic\in\mathcal X_\partic$ with $\|\partic^0-\partic\|\leq \delta$, $F(\partic^0)\leq  F(\partic)$, i.e. 
\[
\risk^\total(\partic^0, \alltheta^0) \leq \risk^\total(\partic, \alltheta^*(\partic))\]
for any minimizing $\alltheta^*(\partic)$.
For fixed allocation $\partic$ define $\risk^\total_{\partic}(\alltheta) = \risk^{\total}(\partic, \alltheta)$ which is
is convex and minimized at $\alltheta^*(\alpha)$ and hence:
$$\risk^\total(\partic, \alltheta^*(\partic)) \leq \risk^\total(\partic, \alltheta), ~~\forall ~~\alltheta \:. $$
Combining the inequalities yields: $\risk^\total(\partic^0, \alltheta^0) \leq \risk^\total(\partic, \alltheta)$, and thus $(\partic^0,\alltheta^0)$ is a local minimum of $\risk^\total$. 
The implication for the isolated local minimum case follows by the same arguments with strict inequalities on the total risk, \propchange{noting that if $\alltheta^0$ is a unique minimizer, it must also be isolated}.

2. $\risk^\total$ minimized $\implies F$ minimized\\
\propchange{Recall that $\risk^\total(\partic,\alltheta)$ can be written as $\tr( \diag(\groupsize) \partic  R(\alltheta)^\top )$.
Then
\[
\risk^\total(\partic^0 + D,\alltheta^0) - \risk^\total(\partic^0,\alltheta^0) = \tr( \diag(\groupsize) D  R(\alltheta^0)^\top ) \geq 0\]
where inequality holds for all $D\in\R^{\numgroup\times\numplayer}$ such that $\partic^0+D\in\mathcal X_\alpha$ by the fact that $\alpha^0$ is a minimum.
Recognizing the gradient from Lemma~\ref{lem:partial_deriv} and using the uniqueness of $\alltheta^0$, the expression is equivalently $\langle \nabla_\alpha F(\alpha^0), D \rangle \geq 0$. In other words, the directional derivative in any feasible direction $D$ is non-negative.
Hence, $\partic^0$ is a local minimum of $F$.
The implication for the isolated local minimum case follows by the same arguments with strict inequalities on the total risk.}

\end{proof}

\begin{lemma}\label{lem:partial_deriv}
For $F:\R^{\numgroup\times\numplayer}\to\R$ defined as in Lemma~\ref{lem:local_min_char}, \propchange{suppose the minimizier}  $\alltheta^*(\alpha)=\arg\min_\alltheta \risk^\total(\alpha, \alltheta)$ \propchange{is unique}.
The gradient is
\[\nabla_\alpha F(\alpha) = \diag(\groupsize)  R\left(\alltheta^*(\alpha)\right),\text{~i.e.} \quad \frac{\partial F(\alpha)}{\partial{\alpha_{ij} }} = {\beta_\indgroup} \risk_\indgroup(\theta^*_\indplayer(\alpha))\:.\]
\propchange{Further suppose that the risks are strongly convex.}
Then second partial derivatives are given by 
\[\frac{\partial^2 F(\alpha)}{\partial \partic_{k\ell}\partial{\alpha_{ij} }}
=\begin{cases} 0 & k\neq j \\-\groupsize_\indgroup\nabla\risk_i(\theta^\star_j)^\top \left(\sum_{\ell'}\beta_{\ell'}\alpha_{\ell'j}\nabla^2\risk_{\ell'}(\theta^\star_j)\right) \nabla \risk_\ell(\theta^\star_j) & k=j \end{cases}\:.\]
\end{lemma}
\begin{proof}

Computing the gradient:
\begin{align*}
    \nabla_\alpha F(\alpha)
&=\nabla_\alpha\risk^{\total}(\alpha, \alltheta^\star(\partic))+\nabla\alltheta^\ast(\alpha)\nabla_{\theta}\risk^{\total}(\alpha, \alltheta^\star(\partic))
= \diag(\groupsize)  R(\alltheta)\:.
\end{align*}
The first equality follows by product rule.
The second equality follows because
1) the total risk is linear in $\alpha$ and 2)
the second term is zero due to the optimality of $\alltheta^*(\partic)$.

Now notice that 
\[\frac{\partial}{\partial \partic_{k\ell}} \risk_i(\theta_j^\star(\partic)) =\left\langle \frac{\partial \theta_j^\star(\partic)}{\partial \partic_{k\ell}} ,\nabla_\theta\risk_i(\theta_j^\star(\partic))\right\rangle \]
To compute the derivatives of $\theta_j^\star(\partic)$ we use the implicit function theorem \propchange{and the assumption that the risks are strongly convex}.
We apply the implicit function theorem to the first order optimality condition
\[\theta_j^\star(\alpha)\in \arg\min_{\theta_j}\bar{\mc{R}}_j^{\sf \player}(\alpha_{:,j},\theta_j)\]
The Hessian
$\nabla^2_\theta\bar{\mc{R}}_j^{\sf \player}(\alpha,\alltheta))$ is non-degenerate due to strong convexity of the \group{} risks. 
There exists a neighborhood $U_0$ of $\alpha$ and a unique (sufficiently smooth) map $\theta^\ast_j(\cdot)$ such that for all $\alpha\in U_0$, we have that $\nabla_\theta \bar{\mc{R}}_j^{\sf\player}(\alpha,\theta^\ast(\alpha))=0$. Then by implicit function theorem  we obtain
\[\nabla \theta_j^\star(\alpha)=-\nabla^2_\theta\bar{\mc{R}}_j^{\sf \player}\circ \nabla_{\alpha\theta}\bar{\mc{R}}_j^{\sf\player}(\alpha_{:,j},\theta^\star_j(\alpha))\]
by taking the derivative of the first order condition differentiating through $\theta_j^\star(\cdot)$ and setting it to zero. 
We have that
\[\nabla^2_\theta\bar{\mc{R}}_j^{\sf \player} = \sum_{\ell'} \beta_{\ell'}\partic_{\ell' j} \nabla^2 \risk_{\ell'}(\theta_j^\star),\quad \frac{\partial }{\partial \alpha_{k\ell}} \nabla_\theta \bar\risk_j^\playerm = \begin{cases}0&k\neq j\\ \nabla \risk_\ell(\theta^\eq_j)&k=j\end{cases}\:.\]
The result follows by combining the expressions. %
\end{proof}

\section{Full Proofs of Main Results}\label{app:results}
In this section, we present proofs of the main results.

\subsection{Connections between dynamics and total risk}
\totalRiskDec*
\begin{proof}[Proof of Proposition~\ref{prop:total_risk_dec}]

The key to seeing that the total risk acts like a potential for the market dynamics is to note two equivalent decompositions of the total risk:
\[\risk^{\total}(\alpha, \alltheta) = \beta^\top \bar\risk^{\groupm}(\partic,\alltheta) = \beta^\top\partic \bar \risk^{\playerm}(\alpha, \alltheta)\:. \]

Being risk-reducing \players{}' updates satisfy: 
\[ \bar \risk^{\playerm}(\alpha^t, \alltheta^{t+1})\leq  \bar \risk^{\playerm}(\alpha^t, \alltheta^t) \implies
\risk^{\total}(\alpha^t, \alltheta^{t+1}) \leq \risk^{\total}(\alpha^t, \alltheta^{t})\:.\]
Similarly risk reducing \groups{} satisfy: \[\bar\risk^{\groupm}(\partic^{t+1},\alltheta^{t+1})\leq \bar\risk^{\groupm}(\partic^t,\alltheta^{t+1}) \implies
\risk^{\total}(\alpha^{t+1}, \alltheta^{t+1}) \leq \risk^{\total}(\alpha^t, \alltheta^{t+1})\:.\]
Finally, combining the two updates yields the desired inequality.

In the case that \players{} and \groups{} are risk minimizing in the limit, the same argument holds with strict inequality, unless $(\alpha^{t+1}, \alltheta^{t+1})$ is a local minimum.
\end{proof}

\eqMinConnection*
\begin{proof}[Proof of Theorem~\ref{prop:eq_min_connection}]

We break this proof into two implications.

1. {Isolated local min $\implies$ Asymptotic stability}\\
Define $V(\partic,\alltheta) = \risk^\total(\partic,\alltheta) - \risk^\total(\partic^\eq,\alltheta^\eq)$.
The dynamics $f$ are Lipschitz by Lemma~\ref{lem:local_lipschitz} and this $V$ satisfies the conditions of Theorem~\ref{thm:1.2} with strict inequality, thus we conclude that $(\partic^\eq,\alltheta^\eq)$ is an asymptotically stable equilibrium.

2. {Not local min $\implies$ Not stable}\\ Define $V(\alpha,\alltheta)= \risk^\total(\partic^\eq,\alltheta^\eq)-\risk^\total(\partic,\alltheta)$ which will increase along trajectories.
Since we are not at a local min, there must be some arbitrarily close $\alpha^0,\alltheta^0$ such that $V(\alpha,\alltheta)>0$. Then we apply Theorem~\ref{thm:1.5} which guarantees that the equilibrium is not stable.
\end{proof}

\eqExistence*
\begin{proof}[Proof of Corollary~\ref{prop:eq_min_connection}]
We first argue that if the dynamics are risk minimizing, then an isolated local minimum of the total risk must be an equilibria.
Let $(\alpha^0, \Theta^0)$ denote the isolated local minima of the total risk.
It must be that $\alpha^0$ is an isolated, and thus unique, minimizer of $\risk^\total(\partic,\alltheta^0)$ since the function is linear in $\alpha$.
We can thus conclude that $\nu(\alpha^{0},\Theta^0) = \alpha^0$.
It also must be that
$\Theta^0$ is a unique minimizer of $\risk^\total(\partic^0,\alltheta)$ since the function is convex in $\Theta$.
We can thus conclude that $\mu(\alpha^{0},\Theta^0) = \Theta^0$.
Therefore $(\alpha^0, \Theta^0)$ is equilibrium of the dynamics.

We next show that equilibria may not exist when the dynamics are not risk minimizing in the limit.
To show that they may not exist otherwise, consider the following example.
Let all learners be static and identical
so $\Theta^{t+1}=\Theta^t$ and $\Theta=(\theta,\theta,\dots,\theta)$.
Let the \group{} update break ties among equivalent \players{} randomly.
Then the \groups{} will randomly switch between learners. 
Though these dynamics satisfy the definition of risk reducing (at equality), they will not converge to an equilibrium.

We lastly show that equilibria may not exist when the total risk function does not have an isolated local minima.
Suppose that learners update with full risk minimization and all subpopulations have risk uniquely minimized at the same value $\theta$.
Finally suppose that subpopulations will break ties among equivalent \players{} randomly (and are otherwise risk minimizing). As in the previous example, the \groups{} will randomly switch between learners and no equilibrium exists.
\end{proof}

\subsection{Stable equilibria}

\segmentedMarketEq*
\begin{proof}[Proof of Theorem~\ref{thm:segmented_market_eq}]
First note that it must be that every \player{} is associated to at least one \group{}.
Otherwise, the total risk would not have a unique minimizer over $\alltheta$.

We start with the first statement, and show that the stated conditions imply that $(\partic^\eq,\alltheta^\eq)$ is isolated local minimum of the total risk. By Theorem~\ref{prop:eq_min_connection}, this implies asymptotic stability.

We specifically argue the conditions are sufficient for guaranteeing an isolated local minimum with respect to $F(\partic)$,
appealing to Lemma~\ref{lem:min_equiv}.
First notice that we have the unique $\alltheta^\eq = \arg\min_\alltheta \risk^\total(\partic^\eq,\alltheta)$ as required. 
Suppose by contradiction that there is some perturbation to $\alpha$ that causes $F(\alpha)$ to decrease or remain the same.
Equivalently,
the projection of the negative gradient onto the simplex points towards some other vertex, i.e. the component of the gradient in the direction of \player{} $j$ is less than or equal to that in the direction of $\gamma(i)$ for some $j\neq \gamma(i)$. 
We can write this condition as
\[\frac{\partial F(\alpha)}{\partial{\alpha_{i\gamma(i)} }}\geq  \frac{\partial F(\alpha)}{\partial{\alpha_{ij} }} \iff \risk_\indgroup(\theta^\eq_{\gamma(\indgroup)}) \geq \risk_\indgroup(\theta^\eq_{\indplayer})\]
where we use Lemma~\ref{lem:partial_deriv}. 
This violates the risk comparison condition~\eqref{eq:nec_as}, and therefore there must be no such perturbation, and thus $\alpha^\eq$ is an isolated local minimum.

We turn our attention to the second statement.
Theorem~\ref{prop:eq_min_connection}, it is equivalent to argue about minima of the total risk function.
Suppose that for some \group{}, there is some \player{} for which $\risk_\indgroup(\theta^\eq_{\gamma(\indgroup)}) > \risk_\indgroup(\theta^\eq_{\indplayer})$.
Then any small perturbation of that \groups{}'s allocation towards that \player{} will decrease the total risk, and thus the point is not a minimum.
\end{proof}

In a segmented allocation, each $\theta_j^\eq$ will minimize the average loss over the group of \groups{} assigned to them.
Denote the parameter which minimizes risk of \group{} $i$ as $\phi_\indgroup := \arg\min_{\theta \in \R^d}\risk_\indgroup(\theta)$.
Then each $\theta_j^\eq$ is a convex combination of $\phi_\indgroup$ for $i$  in $j$th partition.
Using this perspective, we provide an intuitive necessary (but not sufficient) condition
for a class of symmetric risk functions. 
\begin{restatable}{coro}{convexHull}
\label{coro:convex_hull}
{Suppose that risk functions satisfy $\risk_\indgroup(\theta) < \risk_\indgroup(\theta') \Longleftrightarrow \Vert \theta - \phi_{\indgroup} \Vert < \Vert \theta' - \phi_{\indgroup} \Vert$ for $\phi_\indgroup$ the \group{} optimal parameter.}
Then in an asymptotically stable segmented equilibrium, the convex hulls of the grouped \groups{} optimal parameters $\{\phi_\indgroup\}$ are non-intersecting.
\end{restatable}
Applying this Corollary to the example in Figure~\ref{fig:lin_2d_ex}, we see that a segmented equilibrium with \group{} 1 and 3 participating in the same learner cannot be stable.

\begin{proof}[Proof of Corollary~\ref{coro:convex_hull}]

Let $\phi_1, \phi_2, \cdots, \phi_k \in \R^d$ be the optimal decisions for the \groups{} allocated to the first \player{} and $\psi_1, \psi_2, \cdots, \psi_l \in \R^d$ be the optimal decisions for the \groups{} allocated to  the second \player{}. 
Let $\theta_1$ and $\theta_2$ be the decisions of each \player{}. 
Assume that the convex hulls of $\{\phi_i\}_{i=1}^k$ and $\{\psi\}_{i=1}^l$ intersect.
By Lemma \ref{lem:chull}, there exists $i$ such that $\Vert \phi_i - \theta_2\Vert \leq \Vert \phi_i - \theta_1\Vert$.
By the assumption about the risk runctions, this implies $\risk_i(\theta_2) < \risk_i(\theta_1)$.
In other words, there exist a \group{} that would prefer to switch \players{}. Thus by Theorem \ref{thm:segmented_market_eq} these allocation of \groups{} to \player{} is not stable and so the convex hulls must not intersect.
\end{proof}

\balancedEq*
\begin{proof}[Proof of Theorem~\ref{thm:balanced_eq}]

Theorem~\ref{prop:eq_min_connection} shows that an equilibrium cannot be stable if it is not a local minimum of the total risk.
We therefore develop conditions under which an equilibrium point will be a local minimum.
By Lemma~\ref{lem:min_equiv}, 
it is equivalent to argue about the local minima of the concave function $F(\partic)$ over the simplex product $\Delta_\numplayer^{\numgroup}$.
All minima of the total risk will occur on the boundary of
the simplex product, i.e. a face or a vertex.
Since $F$ is still concave when restricted to a face of the simplex, the same argument shows the minima are on the boundary, hence vertices, except for the degenerate case where $F$ takes a constant value over the face.

We now characterize this degenerate case. $F$ takes a constant value over the face if and only if 1) the gradient of $F$ is perpendicular to the face at $\alpha$ and 2) remains perpendicular along the face.
The face is described by a set of indices $\mathcal J \subseteq [\numplayer]$.
Mathematically, we write the two conditions as: for all pairs $j,j'\in\mathcal J$, $\ell\in[\numgroup]$, and $k\in[\numplayer]$,
\begin{align}\label{eq:condition_face}
    \frac{\partial F(\alpha)}{\partial{\alpha_{ij} }} = \frac{\partial F(\alpha)}{\partial{\alpha_{ij'} }}\quad \text{and} \quad \frac{\partial}{\partial \alpha_{\ell k}} \left( \frac{\partial F(\alpha)}{\partial{\alpha_{ij} }} - \frac{\partial F(\alpha)}{\partial{\alpha_{ij'} }}\right)=0
\end{align}

Using  Lemma~\ref{lem:local_min_char}, the first expression simplifies to the \emph{risk equivalent} condition that $\risk_\indgroup(\theta^\eq_j) = \risk_\indgroup(\theta^\eq_{j'})$.
Turning to the second expression in \eqref{eq:condition_face}, we first use Lemma \ref{lem:partial_deriv} to compute
\[\frac{\partial}{\partial \alpha_{\ell k}} \frac{\partial F(\alpha)}{\partial{\alpha_{ij} }} =\begin{cases} 0 & k\neq j \\-\groupsize_\indgroup \nabla\risk_i(\theta^\eq_j)^\top \left(\sum_{\ell'}\beta_{\ell'}\alpha_{\ell'j}\nabla^2\risk_{\ell'}(\theta^\eq_j)\right) \nabla \risk_\ell(\theta^\eq_j) & k=j \end{cases}\]
Thus, the condition trivially holds for $k\notin \{j,j'\}$.
Otherwise, when $\ell=i$, the condition in \eqref{eq:condition_face} requires that
\[\nabla\risk_i(\theta^\eq_k)^\top \left(\sum_{\ell'}\beta_{\ell'}\alpha_{\ell'k}\nabla^2\risk_{\ell'}(\theta^\eq_k)\right) \nabla \risk_i(\theta^\eq_k) = 0,\quad k\in\{j,j'\}\]
Due to the strong convexity of the risks, the
Hessian matrix is positive definite. 
Thus it must be that $\nabla \risk_i(\theta^\eq_j)=0$ for all $j\in\mathcal J$, i.e. the \emph{risk optimal} condition. Risk optimality implies that the condition holds also when $\ell\neq i$ and thus the characterization is complete.
\end{proof}

\subsection{Social Welfare}
\begin{proof}[Proof of Proposition~\ref{prop:socialwelf}]
Social welfare is non-decreasing (or increasing) if and only if total risk is non-increasing (or decreasing), as guaranteed by Proposition~\ref{prop:total_risk_dec}.
Maxima of the social welfare are equivalent to minima of the total risk and therefore the connections to stable equilibria follow by Theorem~\ref{prop:eq_min_connection}.
\end{proof}

\begin{proof}[Proof of Proposition \ref{prop:add_player}]
By construction $(\tilde\partic^\eq, \tilde\alltheta^\eq)$ is not segmented, and neither is it a stable balanced equilibrium (by the non-optimality assumption).
Therefore, it is not stable (Theorem~\ref{thm:balanced_eq}), and thus not a local minimum of the total risk (Theorem~\ref{prop:eq_min_connection}). A perturbation will thus send the system along a risk-reducing trajectory.
\end{proof}

\section{Detailed Examples}

\subsection{Risk Reducing and Minimizing Dynamics}\label{sec:appendix}

\ExampleRR*

\begin{proof}
To see that the \group{} is risk minimizing, first see that
\begin{align*}
   \bar\risk^{\groupm}_\indgroup \left(\groupalpha^{t+1}, \alltheta \right) 
   &=\sum_{j=1}^m \alpha_{ij}^{t+1}\risk_i(\theta_j)\\
   &= \sum_{j=1}^m \frac{\partic_{\indgroup\indplayer}^t\cdot \exp(-\gamma  \risk_\indgroup(\theta_\indplayer))}{{\sum_{j=1}^m\partic_{\indgroup\indplayer}^t\cdot \exp(-\gamma  \risk_\indgroup(\theta_\indplayer))}} \risk_i(\theta_j)\\
   &< \sum_{j=1}^m \alpha_{ij}^{t}\risk_i(\theta_j) = \bar\risk^{\groupm}_\indgroup \left(\groupalpha^{t}, \alltheta \right) 
\end{align*}
where the strict inequality holds
as long as $\alpha_{ij}^t$ is not on the boundary of the simplex.
Second, observe that for a fixed $\Theta$, $\alpha_{ij}^t\to 1$ if and only if $\risk_i(\theta_j)$ is minimal over all \players{} for which $\alpha_{ij}^0>0$.

To see that the \player{} is risk minimizing, notice that the gradient update is equivalently
$$\theta_j^{t+1} =  \theta_j^t - \gamma_t \nabla_{\theta}  \bar\risk^{\playerm}_\indplayer\left(\playeralpha, \theta_{\indplayer} \right) \:. $$
Gradient descent on an $L$-smooth and convex function leads to strictly decreasing objective values when $\theta_j^t$ is not at a minimum and the step size satisfies $\gamma^t<\frac{2}{L}$. It further converges to a minimum in the limit as long as the step size satisfies the Robbins-Munroe condition (see, e.g. \cite{liu2022almost,blogpost}).
\end{proof}

\begin{example}[Non-continuity of allocation updates]
Suppose a population prefers one learner over others, and only shifts participation away from the preferred learner if there is another with risk smaller by at least $R_0>0$. This is risk reducing but not minimizing in the limit.
\end{example} 

\begin{example}[Shifting to lower-risk models]
 If a \group{}'s allocation updates always shift allocation from  \players{} with high \group{} risk to \players{} with lower \group{} risk, then the allocation is risk reducing.  It may or may not be risk minimizing in the limit.
\end{example}

\begin{example}[Allocations determined by gradient descent]
Consider an allocation determined by (projected) gradient descent with respect to a \group{}'s average risk. This is risk-reducing, and may be risk minimizing in the limit depending on the step-size.
\end{example}

\subsection{Stability}\label{sec:stability ex}

To illustrate the subtleties of determining stability when the total risk function has non-isolated local minima, %
we consider a setting with $n=m=2$ \groups{} and \players{} where $\risk_1(\theta)=\risk_2(\theta) = \theta^2$.
Then the total risk function is minimized for any $\alpha\in\Delta_m^n$ and $\Theta=(0,0)$.
This continuum of minima can contain equilibria of risk minimizing dynamics, and those equilibria may be stable, asymptotically stable, or unstable, which we illustrate with the following examples.

\begin{example}[Continuum of stable balanced markets] \label{ex:s_bm}
    Suppose that \groups{} update their allocation via any Lipschitz continuous risk minimizing update rule which is stationary whenever \players{} are risk equivalent (i.e. $\mc{R}_i(\theta_1) = \mc{R}_i(\theta_2)$).
    Suppose that learners update via full risk minimization.
    Then equilibria will have the form $(\alpha^\eq,(0,0))$ for any $\alpha^\eq\in\Delta_2^2$.
    
    Then starting from any $(\alpha^0,\Theta^0)$ with a $\delta_\alpha, \delta_\theta$ ball of any equilibrium $(\alpha^\eq,\Theta^\eq)$,
    \[\alpha^1 = \nu(\alpha^0, \Theta^0),\quad \Theta^1 = (0,0)\]
    at which point the system is in a new equilibrium, since any allocation $\alpha$ is a fixed point when $\Theta=(0,0)$ so $\alpha^t=\alpha^1$ and $\Theta^t=\Theta^1$ for all $t$.
    We have that $\|\Theta^\eq-\Theta_0\|=0$ and
    \[\|\alpha^\eq-\alpha^1\| = \|\nu(\alpha^\eq,\Theta^\eq)- \nu(\alpha^0, \Theta^0)\| \]
    By the assumption of Lipschitzness, this distance will scale linearly in $\delta_\alpha,\delta_\theta$
    so the definition of stability is satisfied for $\delta$ chosen proportionally to $\epsilon$ depending on the Lipschitz constant of $\nu$.
    
    In this example, any perturbation converges to a new fixed point within one time step.
The continuity of the update functions ensures that the new fixed point is within a bounded distance of the original, satisfying the definition of stability. This example is not asymptotically stable: the allocation does not convergence back to the original point.
\end{example}

\begin{example}[Asymptotically stable segmentation] \label{ex:as_sm}
Consider the \group{} and \player{} update rules as in the prior example, with one amendment. When $\mc{R}_i(\theta_1) = \mc{R}_i(\theta_2)$, \group{} 1 re-allocates half of its mass from learner $2$ to learner $1$, and while \group{} 2 re-allocates half its mass from learner $1$ to learner $2$.
Thus the \group{} update can be written as
\[
\alpha_{1,:}^{t+1} = \begin{cases}  \nu_1(\alpha_{1,:}^t,\Theta^t) & \mc{R}_1(\theta_1) \neq \mc{R}_1(\theta_2)\\
 \begin{bmatrix}1&1/2\\ 0&1/2\end{bmatrix}   \alpha_{1,:}^{t} & \mc{R}_1(\theta_1) = \mc{R}_1(\theta_2)
\end{cases},\quad 
\alpha_{2,:}^{t+1} = \begin{cases}  \nu_2(\alpha_{2,:}^t,\Theta^t) & \mc{R}_2(\theta_1) \neq \mc{R}_2(\theta_2)\\
 \begin{bmatrix}0&1/2\\ 1&1/2\end{bmatrix}  \alpha_{2,:}^{t+1} & \mc{R}_2(\theta_1) = \mc{R}_2(\theta_2)
\end{cases}\]
The only equilibrium has $\alpha^\eq$ segmented with \group{} $i$ associated to \player{} $i$ for $i=1,2$ and $\Theta^\eq=(0,0)$. It is straightforward to see that this is an asymptotically stable equilibrium, since for any $a\in \Delta_2$,
\[\lim_{t\to\infty} \begin{bmatrix}1&1/2\\ 0&1/2\end{bmatrix}^t a = \begin{bmatrix}1&1\\ 0&0\end{bmatrix}a = \begin{bmatrix}1\\0\end{bmatrix}
\quad\text{and}\quad 
\lim_{t\to\infty} \begin{bmatrix}0&1/2\\ 1&1/2\end{bmatrix}^ta = \begin{bmatrix}0&0\\ 1&1\end{bmatrix}a = \begin{bmatrix}0\\1\end{bmatrix}.\]
\end{example}

\begin{example}[Asymptotically stable balanced market]\label{ex:as_bm}
Consider a setting similar to the previous example
except that when $\mc{R}_i(\theta_1) = \mc{R}_i(\theta_2)$, \group{} $i$ moves half the mass from group $1$ to group $2$ and half the mass from group $2$ to group $1$ for all $i$.
Then the \group{} update can be written as \[
\alpha_{1,:}^{t+1} = \begin{cases}  \nu_1(\alpha_{1,:}^t,\Theta^t) & \mc{R}_1(\theta_1) \neq \mc{R}_1(\theta_2)\\
 \begin{bmatrix}1/2&1/2\\ 1/2&1/2\end{bmatrix}   \alpha_{1,:}^{t} & \mc{R}_1(\theta_1) = \mc{R}_1(\theta_2)
\end{cases},\quad 
\alpha_{2,:}^{t+1} = \begin{cases}  \nu_2(\alpha_{2,:}^t,\Theta^t) & \mc{R}_2(\theta_1) \neq \mc{R}_2(\theta_2)\\
 \begin{bmatrix}1/2&1/2\\ 1/2&1/2\end{bmatrix}  \alpha_{2,:}^{t+1} & \mc{R}_2(\theta_1) = \mc{R}_2(\theta_2)
\end{cases}\]
The only equilibrium has $\alpha_{i,:}^\eq=[1/2, 1/2]$ for $i=1,2$ and $\Theta^\eq=(0,0)$. It is straightforward to see that this is an asymptotically stable equilibrium, since for any $a\in\Delta_2$,
\[\lim_{t\to\infty} \begin{bmatrix}1/2&1/2\\ 1/2&1/2\end{bmatrix}^ta = \begin{bmatrix}1/2&1/2\\ 1/2&1/2\end{bmatrix}a = \begin{bmatrix}1/2\\1/2\end{bmatrix}.\]
\end{example}

\begin{example}[Unstable balanced market] \label{ex:u_bm}
Suppose that \group{} allocations follow a projected gradient descent update for all $i$:
\[\alpha_{i1}^{t+1} = \mathrm{Proj}_{[0,1]}\left(\alpha_{i1}^t - \gamma(\mc{R}_i(\theta_1) - \mc{R}_i(\theta_2))\right)\]
and $\alpha_{i2} = 1-\alpha_{i1}$.
Further suppose \players{} update with gradient descent: 
\[\theta_j^{t+1} =
\theta_j^{t} - \frac{1}{2\sqrt{t}} \nabla \bar\risk_j^{\playerm}(\alpha^t_{:,j}, \theta_j^t)
=
\sqrt{\frac{t}{t+1}}\theta_j^t\]
Both rules are risk minimizing in the limit (note that $\theta_j^t = \frac{1}{\sqrt{t}}\theta_j^0$) and have a continuum of equilibria at any $\alpha^\eq \in \Delta_m^n$ and $\Theta^\eq = (0,0)$.
However, we show that the equilibria are not stable.
Consider the initial condition $(\alpha^\eq, (\delta_\theta, 0))$.
We have that
\[\alpha_{i1}^{t+1} 
=\mathrm{Proj}_{[0,1]}\left(\alpha_{i1}^0 - \gamma\delta_\theta^2\sum_{k=1}^t\frac{1}{k} \right) \to 0 \quad\text{as}\quad t\to\infty.\]
No matter how small the perturbation $\delta_\theta$ is, the summation increases with $t$ and participation will converge all weight to learner 2.
A similar argument shows that perturbations exist that will send all participation to learner 1.

In this example, the learners update slowly. 
Despite eventual convergence to the minimizing parameter, 
the accumulating error causes the participation allocation to shift completely to the unperturbed learner, precluding stability.
\end{example}

\subsection{Social Welfare}

We begin with a somewhat generic example with $\numplayer=2$ and $\numgroup=3$
 that illustrates the difference between stable equilibria and social welfare optima.
\begin{example}[Stability vs. optimality] \label{ex:analytic_acmples}
Consider three \groups{} $i\in\{1,2,3\}$ with risks $\|\theta-\phi_i\|_2^2$, sizes $\beta_i$, and two \players{} $j\in\{1,2\}$.
Suppose that the $\partic^\eq$ is such that the \groups{} are partitioned into $\{1\}$ and $\{2,3\}$.
Then we have that
\[\theta_1^\eq = \phi_1,\quad \theta_2^\eq =\frac{\beta_2}{\beta_2+\beta_3}\phi_2 + \frac{\beta_3}{\beta_2+\beta_3}\phi_3\]
By Theorem~\ref{thm:segmented_market_eq}, this is stable if and only if
\[\|\phi_2-\phi_3\|_2 \leq (\beta_2+\beta_3)\min\left\{ \frac{\|\phi_2 - \phi_1\|_2}{\beta_3},  \frac{\|\phi_3 - \phi_1\|_2}{\beta_2}\right\}\:. \]
However, it is only social optimal if and only if $\phi_2$ and $\phi_3$ are relatively close to each other than to $\phi_1$, i.e.
\[\|\phi_2-\phi_3\|_2 \leq \min\left\{\|\phi_2 - \phi_1\|_2, \|\phi_3 - \phi_1\|_2\right\}\:. \]
The set of \group{} optima $\{\phi_1,\phi_2,\phi_3\}$ satisfying the optimality condition are a subset of those satisfying the stability condition. 
As the difference between $\beta_2$ and $\beta_3$ becomes more extreme, the number of settings satisfying the stability but not optimality condition increases.
\end{example}

We use this generic example to illustrate a scenario in which the total risk can be arbitrarily high at a stable equilibria.

\begin{example}\label{ex:local-vs-global}
Suppose there are two \players{} and three \groups{} with sizes $\beta_1=\beta_2=\beta$ and $\beta_3=1-2\beta$ for some $0<\beta<1/2$.
Consider the following:
$\risk_1(\theta) = \theta^2$, $\risk_2(\theta) = (\theta-1)^2$, $\risk_2(\theta) = (\theta-\frac{1-\beta}{1-2\beta} + \epsilon)^2$.
The  social welfare optimizing decision $\alltheta^\star=(1/2,\frac{1-\beta}{1-2\beta} - \epsilon)$ corresponds to total risk $\beta/2$. However, there is a stable equilibrium at $\alltheta^\eq=(0,1 + \epsilon)$ with total risk  $\beta+\frac{(\beta-\epsilon)^2}{1-2\beta}$. For $\beta\to1/2$, the gap becomes arbitrarily large.
\end{example}

Finally, we present an example which illustrates that even in the single \player{} setting, the risk of a \group{} can be arbitrarily worse than the total risk.

\begin{example}[Arbitrarily high risk for minority \group{}]\label{ex:minority-worst-case}
Consider two \groups{} with
$\risk_1(\theta) = \theta^2$ and $\risk_2(\theta) = (\theta-\phi)^2$ with $\beta_1=\beta$ and $\beta_2=1-\beta$ and a single \player{}.
The single equilibrium and total risk minimizer is $\theta_1=(1-\beta)\phi$ with total risk 
$\beta (1-\beta)\phi^2$ and $\risk_2(\theta^\star) = \beta^2\phi^2$. The difference between the two quantities can be arbitrarily high as $\beta$ gets close to 1.
\end{example}
 \section{Additional Experiments with Noisy Dynamics}\label{app:noisy_dynamics}
 Figure \ref{fig:subfig-a} replicates Fig. \ref{fig:cost_curves} from the main text. The magenta-highlighted trajectory starts precisely at the unstable equilibrium, while the other three, initiated near this point, converge to the three possible split market equilibria, ordered by hue intensity: \{(1,2), (3)\}, \{(2,3), (1)\}, and \{(1,3), (2)\}.  In Figure \ref{fig:subfig-b}, while sub-population dynamics remain as in (a), learner updates experience uncorrelated external perturbations, causing trajectories to be different from (a). Nevertheless, the long term dynamics gravitate near stable split equilibria. Figure \ref{fig:subfig-c} depicts learners updating decisions based on sampled empirical losses, with sub-populations adjusting participation based on aggregate empirical performance. The fact that each learner uses different samples from each sub-population adds sufficient un-correlated noise to create trajectories similar to when exogenous noise is added.
\begin{figure}[h]
    \centering
    \begin{subfigure}{0.95\textwidth}
    \centering
        \includegraphics[width=0.6\textwidth]{figures/risk_dynamics_0.pdf}
        \caption{Learner updates: \textbf{noiseless} one-step minimization of \textbf{population} loss. Sub-population updates: MWU w.r.t \textbf{population} loss}
        \label{fig:subfig-a}
    \end{subfigure}
    \hfill
    \begin{subfigure}{0.95\textwidth}
    \centering
        \includegraphics[width=0.6\textwidth]{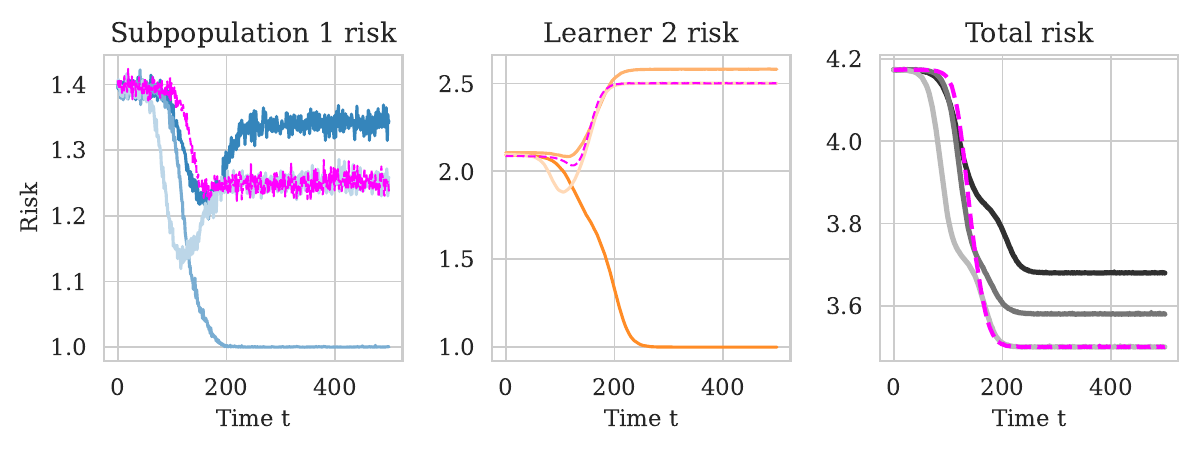}
        \caption{Learner updates: \textbf{noisy} one-step minimization of \textbf{population} loss. Sub-population updates: MWU w.r.t \textbf{population} loss}
        \label{fig:subfig-b}
    \end{subfigure}
    \begin{subfigure}{0.95\textwidth}
    \centering
        \includegraphics[width=0.6\textwidth]{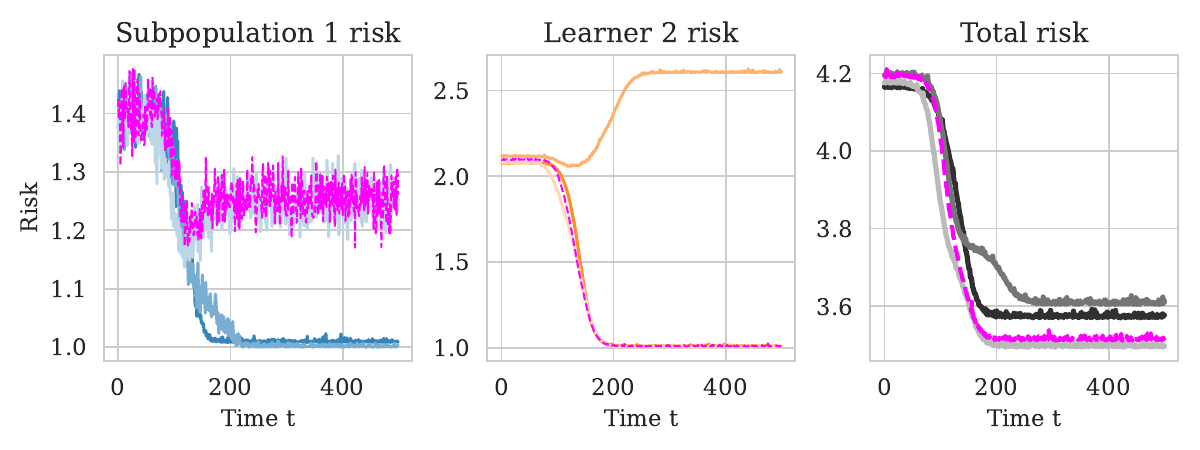}
        \caption{Learner updates: \textbf{noiseless} one-step minimization of \textbf{empirical} loss. Sub-population updates: MWU w.r.t \textbf{empirical} loss}
        \label{fig:subfig-c}
    \end{subfigure}
    \caption{Noisy dynamics}
    \label{fig:main}
\end{figure}
\section{Experimental Details}
Full experimental details along with instructions for reproducing them can be found at 
\url{https://github.com/mcurmei627/MultiLearnerRiskReduction}.
The experiments used Python 3.10 on a MacBook Pro 2019.

\end{document}